\documentclass{article}

\usepackage[nonatbib,preprint]{neurips_2022}

\usepackage[style=alphabetic,backend=bibtex,maxcitenames=6,maxalphanames=6,maxbibnames=99]{biblatex}
\newcommand{\citep}{\cite}
\newcommand{\citet}{\cite}

\usepackage{tikz-qtree}

\usepackage{amsmath}
\usepackage{amsthm}
\usepackage{amssymb}
\usepackage{physics}
\usepackage{algorithm}
\usepackage{algpseudocode}
\usepackage{thmtools}
\usepackage{thm-restate}
\usepackage{bbm}
\usepackage{enumitem}

\usepackage{tikz}
\usepackage{tikz-cd}
\usepackage{float}
\usetikzlibrary{automata, arrows}
\usetikzlibrary{positioning}

\newcommand{\cA}{\mathcal{A}}

\newcommand{\cF}{\mathcal{F}}

\newcommand{\cH}{\mathcal{H}}

\newcommand{\cM}{\mathcal{M}}

\newcommand{\cO}{\mathcal{O}}
\newcommand{\cP}{\mathcal{P}}

\newcommand{\cR}{\mathcal{R}}
\newcommand{\cS}{\mathcal{S}}
\newcommand{\cT}{\mathcal{T}}

\newcommand{\cX}{\mathcal{X}}

\newcommand{\bE}{\mathbb{E}}

\newcommand{\bN}{\mathbb{N}}

\newcommand{\bR}{\mathbb{R}}

\renewcommand{\phi}{\varphi}
\newcommand{\neval}{n_{\texttt{eval}}}
\newcommand{\nroll}{n_{\texttt{rollout}}}

\newcommand{\epseval}{\varepsilon_{\texttt{eval}}}
\newcommand{\bepseval}{\bar{\varepsilon}_{\texttt{eval}}}

\newcommand{\epstol}{\varepsilon_{\texttt{tol}}}
\newcommand{\epsroll}{\varepsilon_{\texttt{roll}}}
\newcommand{\epstarget}{\varepsilon_{\texttt{target}}}

\usepackage{empheq}
\usepackage[most]{tcolorbox}
\newtcbox{\mymath}[1][]{%
    nobeforeafter, math upper, tcbox raise base,
    enhanced, colframe=blue!30!black,
    colback=blue!30, boxrule=1pt,
    #1}

\DeclareMathOperator{\argmin}{argmin}
\DeclareMathOperator{\argmax}{argmax}

\DeclareMathOperator{\Dist}{Dists}

\DeclareMathOperator{\poly}{\text{poly}}

\DeclareMathOperator{\Ball}{\texttt{Ball}}
\DeclareMathOperator{\Ber}{\texttt{Ber}}

\newtheorem{theorem}{Theorem}[section]

\newtheorem{lemma}[theorem]{Lemma}

\newtheorem{definition}[theorem]{Definition}

\newtheorem{assumption}[theorem]{Assumption}

\DeclareMathOperator{\reset}{\textsc{reset}}
\newcommand{\measuretd}{\texttt{measureTD}}
\DeclareMathOperator{\query}{\textsc{oracle}}
\newcommand{\Delphi}{\textsc{Delphi }}
\newcommand{\Delphinospace}{{\textsc{Delphi}}}
\DeclareMathOperator{\CubeGame}{\textsc{CubeGame}}
\DeclareMathOperator{\MDP}{\textsc{MDP}}
\DeclareMathOperator{\Wclose}{W_{\text{close}}}
\newcommand{\ctflip}{\text{ct}^\text{flip}}
\newcommand{\fix}{\text{fix}}
\newcommand{\efix}{e^\text{fix}}
\newcommand{\enfix}{e^{\neg \text{fix}}}

\usepackage[utf8]{inputenc} %
\usepackage[T1]{fontenc}    %
\usepackage{hyperref}       %
\usepackage{url}            %
\usepackage{booktabs}       %
\usepackage{amsfonts}       %
\usepackage{nicefrac}       %
\usepackage{microtype}      %
\usepackage{xcolor}         %

\title{A Few Expert Queries Suffices for Sample-Efficient RL with Resets and Linear Value Approximation}

\newcommand{\affilone}{\varphi}
\newcommand{\affiltwo}{\psi}

\author{%
  {Philip Amortila$^\affilone$}\thanks{\texttt{philipa4@illinois.edu}} \hspace{1cm} Nan Jiang$^\affilone$ \hspace{1cm} Dhruv Madeka$^\affiltwo$ \hspace{1cm} Dean P. Foster$^\affiltwo$ \\
     $^\affilone$University of Illinois, Urbana-Champaign \hspace{1cm} $^\affiltwo$Amazon
}

\begin{document}

\maketitle

\begin{abstract}

The current paper studies sample-efficient Reinforcement Learning (RL) in settings where only the optimal value function is assumed to be linearly-realizable. It has recently been understood that, even under this seemingly strong assumption and access to a generative model, worst-case sample complexities can be prohibitively (i.e., exponentially) large. We investigate the setting where the learner additionally has access to interactive demonstrations from an expert policy, and we present a statistically and computationally efficient algorithm (\Delphinospace) for blending exploration with expert queries. In particular, \Delphi requires $\tilde\cO(d)$ expert queries and a $\texttt{poly}(d,H,|\cA|,1/\varepsilon)$ amount of exploratory samples to provably recover an $\varepsilon$-suboptimal policy. Compared to pure RL approaches, this corresponds to an exponential improvement in sample complexity with surprisingly-little expert input. Compared to prior imitation learning (IL) approaches, our required number of expert demonstrations is independent of $H$ and logarithmic in $1/\varepsilon$, whereas all prior work required at least linear factors of both in addition to the same dependence on $d$. Towards establishing the minimal amount of expert queries needed, we show that, in the same setting, any learner whose exploration budget is \textit{polynomially-bounded} (in terms of $d,H,$ and $|\cA|$) will require \textit{at least} $\tilde\Omega(\sqrt{d})$ oracle calls to recover a policy competing with the expert's value function. Under the weaker assumption that the expert's policy is linear, we show that the lower bound increases to $\tilde\Omega(d)$.

\end{abstract}

\section{Introduction}

Many potential applications of reinforcement learning (RL) have intractably-large state-spaces. Thus, we seek provably-correct methods which have statistical and computational requirements that are independent of the size of the state-space. This requires some modelling assumptions. One dominating approach has been to introduce \textit{function approximation}, and to posit that the MDP or its value functions are well-represented by the function approximation scheme which is employed. A basic starting point which still lacks comprehensive understanding is the case of linear value function approximation, which models value functions as lying in the span of a known $d$-dimensional feature mapping and asks for methods which have sample complexities that are polynomial only in $d, H,$ and possibly $|\cA|$ ($H$ and $\cA$ are the horizon and action sets of the MDP, respectively). This desideratum was recently understood to be impossible for the ``minimal'' case where only the optimal value function (or optimal action-value function) is assumed to be linear -- i.e. there exist MDPs satisfying this assumption where the statistical complexity of any algorithm will be exponentially large, either in $d$ or in $H$
\citep{weisz2021exponential,weisz2021tensorplan,wang2021exponential,foster2021statistical}. Furthermore, this also holds in the case where the learner is equipped with a generative model (or simulator), allowing them to sample transitions from any state of the MDP. In recent years much has been said about linear value approximation under stronger assumptions,  for example under determinism \citep{Wen_Roy_2013}, linear/low-rank MDPs \citep{Jin_Yang_Wang_Jordan_2019,ayoub2020model}, Bellman-closedness \citep{LaSzeGe19,zanette2020learning} , or the existence of a ``core set'' \citep{shariff2020efficient,zanette2019limiting}. These stronger assumptions can recover polynomial statistical complexities (if not computational ones), but are restrictive and oftentimes unrealistic. 

In this work, we consider an alternative possibility for recovering polynomial sample complexities which do not further restrict the class of MDPs under consideration. 
That is, we assist the learner with some additional side information about the problem. Specifically, we assume that there is a deterministic expert policy (which need not be the optimal policy) that the learner can query at any state, whereupon they will be informed of the expert's action at that state. Indeed, such information can often be made readily available if we have some form of prior knowledge (or human input) about the problem. 
Leveraging such expert demonstrations has been studied in interactive imitation learning (IL), with common applications in simulated domains \citep{ross2011reduction, ross2013interactive, ross2014reinforcement,sun2017deeply}. As we will see, however, the amount of expert queries required by a pure IL approach is significantly higher than what we need. Since interaction with a (human) expert might be costly, we wish to minimize the burden of the expert by having the learner explore mostly on their own, and only query the expert in a judicious manner. %
The question asked by this work, then, is: 
\begin{quote}
    \textit{Under linear-realizability, what is the \textit{minimal} amount of expert data required for sample/computational-efficient learning, and which algorithm achieves this?}
\end{quote} 

Our main result is the \Delphi algorithm for exploring with an interactive expert. 
\Delphi assumes that the expert's value function is linear, and that the agent has access to a $\reset$ function which lets them return to the state \textit{most recently seen}. Under these conditions, our method uses surprisingly-few expert queries combined with some modest (polynomial) amount of exploration to recover the expert policy. Formally, \Delphi recovers a policy matching which is $\varepsilon-$optimal (with respect to the expert policy) with $\cO(d \log (B/\varepsilon))$ oracle calls and $\tilde\cO(\frac{d^2H^5|\cA|B^4}{\varepsilon^4})$ exploratory samples, where $B$ is a bound on the $\ell_2$-norm of the unknown linear parameter.\footnote{The $\tilde\cO$ notation ignores logarithmic factors.}
Thus, our results show that merely $\tilde\cO(d)$ expert queries enable an \textbf{exponential improvement} in sample complexity when compared to RL without expert advice. Furthermore, the number of oracle calls is completely independent of the horizon of the problem, whereas prior work in IL leveraging similar expert advice requires (at best) linear factors of $H$ in addition to scaling with $d$. We also show that $\Delphi$ is computationally efficient, that it is robust to some misspecification error, and that it can be extended to the case where the \textit{action}-value function of the expert is linear when the MDP dynamics are deterministic.

Towards establishing the optimality of our algorithm, we study the capabilities of expert-augmented learners which have fixed exploration budgets. More specifically, we ask: what is the minimal number of expert queries \textit{required} by any algorithm which is constrained to a \textit{polynomially-bounded} exploration budget? %
We show that any polynomially-bounded learner (in terms of $d,H,$ and $|\cA|$) will require at least $\tilde\Omega(\sqrt{d})$ oracle calls to recover a policy competing with the expert's value function. In the more relaxed setting where only the expert's \textit{policy} is linear, we show that this lower bound increases to $\tilde\Omega(d)$, matching our upper bound up to logarithmic factors.

The rest of the paper is structured as follows: in Section \ref{sec:background} we review background and present the problem setting. Section \ref{sec:delphi} describes our algorithm, its guarantee, a sketch of the proof, and discusses some extensions. Section \ref{sec:lb} studies the minimal amount of expert queries needed. We conclude with an overview of related work and some discussion in Sections \ref{sec:prior} and \ref{sec:conc}.

\section{Background \& Problem Setting}\label{sec:background}

 \paragraph{Notation} We write $\Dist(\cX)$ for the set of probability measures on some set $\cX$. We write $[N] \coloneqq \{1, \dots, N\}$.
The \textit{direct product} $\oplus$ corresponds to ``concatenating'' two vectors, i.e. for any two vectors $u \in \bR^n$ and $v \in \bR^m$, we have $u \oplus v = (u_1,\dots,u_n, v_1, \dots, v_m)^\top \in \bR^{n+m}$. We write $\otimes$ for the tensor (or outer) product of two vectors, defined by $u \otimes v = u v^\top \in \bR^{n \times m}$, and $\flat( u \otimes v) \in \bR^{n \cdot m}$ for the flattening (or vectorization) of said tensor product.

\paragraph{MDPs, Policies, and Value Functions} The typical environment in RL is modelled as an MDP \citep{puterman2014markov,szepesvari2010algorithms,sutton2018reinforcement}. We consider here finite-horizon MDPs, which are specified by a tuple  $\cM = \left( \cS, \cA, \cR, \cP, H, \mu_0\right)$, where $\cS$ is a state space, $\cA = [A]$ is a finite action set, $\cR: \cS \times \cA \rightarrow \Dist([0,1])$ is a (bounded) reward distribution function with expectation $r(s,a)$, $\cP: \cS \times \cA \rightarrow \Dist(\cS)$ is the transition distribution function with probability vectors $P(s,a) = [ \cP(s'|s,a)]_{s' \in \cS} \in \bR^{|\cS|}$, $H \in \bN$ is the horizon, and $\mu_0 \in \Dist(\cS)$ is the starting distribution. Note that we have assumed that the action space is finite, although the state space may be infinite. Without loss of generality we assume that $\cS$ is a disjoint union of per-horizon state spaces, i.e. $\cS = \cup_{h \in [H]} \cS_h$. 

A (non-stationary) policy prescribes a sequence of actions
$\pi : \cS_h \rightarrow \Dist(\cA)$, and its \textit{value function} is
$
    v^{\pi}(s) =\mathbb{E}[\sum_{h'=h}^H r(s_{h'},a_{h'}) \mid s_h = s, a_{h'} \sim \pi(s_{h'})],
$ 
where $s \in \cS_h$. The \textit{action-value} function $q^{\pi}(s,a)$ is defined similarly, save that the first action taken is $a$ and the proceeding actions follow $\pi$. 
Value functions satisfy the recursive relationship:
\begin{align}  \label{eq:bellman} 
    v^{\pi}(s) = r(s,\pi) + \langle P(s,\pi), v^\pi(\cdot) \rangle
            \coloneqq \cT^{\pi} v^\pi(s)
\end{align}
where we have the shorthands $r(s,\pi) = \bE_{a \sim \pi(s)}[r(s,a)]$, $P(s,\pi) = \bE_{a \sim \pi(s)}[P(s,a)]$, and the Bellman operator $\cT^\pi (\cdot) \coloneqq r(s,\pi) + \langle P(s,\pi), (\cdot) \rangle$. The Bellman operator has $v^\pi$ as its unique fixed point. The optimal policy is written $\pi^\star$, and its value function is denoted as $v^\star \coloneqq v^{\pi^\star}$. The objective is to find a $\pi$ maximizing $v^{\pi}(\mu_0) \coloneqq \bE_{s_0 \sim \mu_0}[ v^\pi(s_0)]$.

\paragraph{Function Approximation With an Interactive Expert}
 In the RL setting, the MDP is unknown and must be explored. As stated in the introduction, we seek sample complexities which are independent of the number of states. This is evidently not possible without further assumptions. In linear value approximation the leaner is provided with a \textit{feature mapping} $\varphi: \cS \rightarrow \bR^d$ used to approximate value functions linearly, i.e. $v_\theta(s) = \langle \theta, \phi(s) \rangle, $ for $\theta \in \bR^d$. 
 
To assist the learner, we assume that the agent has further access to an \textit{oracle}, which, upon being queried, returns the action of an 
\textit{expert} policy $\pi^\circ$ for the state. The expert policy need not be the optimal policy.  %
We will assume for simplicity that $\pi^\circ$ is deterministic. 

\begin{assumption}[Interactive expert]\label{ass:inter-exp}
There is an \textit{oracle} which can be queried at any state $s$, which returns an action $\pi^\circ(s)$. Syntactically, the oracle is queried via the $\query(s)$ function.
\end{assumption}
The objective, then, is to recover a policy which competes with the expert policy, namely a $\hat\pi$ such that:
\begin{equation}\label{eq:pac-expert}
v^{\hat{\pi}}(\mu_0) \geq v^\circ(\mu_0) - \varepsilon \, \quad \text{with probability } \geq 1-\delta,
\end{equation}
where $v^\circ \coloneqq v^{\pi^\circ}$ is the value function of the expert. In the sequel we refer to a policy satisfying Equation \ref{eq:pac-expert} as \textit{$\varepsilon$-optimal}. To aid in this objective, our next assumption is that the expert's value function is linear in a set of known features. 

\begin{assumption}[$v^\circ$-linearity, with bounded features]\label{ass:exp-lin}
The value function $v^\circ$ \textit{of the expert} is linear with known features $\varphi$, i.e.
\begin{equation}\label{eq:lin-exp}
    v^\circ(s) = \langle \phi(s) , \theta^\circ \rangle, ~ \forall s \in \cS,
\end{equation}
for some unknown $\theta^\circ \in \bR^d$. We further assume that $\norm{\phi(s)}_2 \leq 1 \, \forall s$ and that $\norm{\theta^\circ}_2 \leq B$ for some known $B \in \bR^d$. 
\end{assumption}
Our last assumption is that the agent has the ability to ``reset'' to the state just-experienced, formally:
\begin{assumption}[Resets]\label{ass:reset}
After experiencing a transition $(s,a,r,s')$ in the MDP, the agent can return to the state $s$. Syntactically, this is done via the $\reset()$ function.
\end{assumption}
As noted in the introduction, Assumptions \ref{ass:exp-lin} and \ref{ass:reset} together are not enough to enable sample-efficient learning, as shown by existing exponential lower bounds. Thus, any algorithm for this setting must necessarily make use of Assumption \ref{ass:inter-exp}. %
Our $\reset$ assumption is also weaker than full generative model access \citep{kearns2002sparse} or the ``local'' simulator setting \citep{weisz2021query,li2021sample,hao2022confident} which has appeared in prior works. %

 \section{The \Delphi algorithm}\label{sec:delphi}

We are ready to describe our approach and give the main result. We begin by supposing that the starting distribution is deterministic (we will see later that this comes at no loss of generality). 

\begin{theorem}\label{thm:delphi-ub}
Suppose Assumptions \ref{ass:inter-exp}, \ref{ass:reset}, and \ref{ass:exp-lin} hold. Then the \Delphi algorithm will recover a policy $\hat{\pi}$ such that
$v^{\hat \pi} (s_0) \geq v^\circ(s_0) - \varepsilon$  with probability $\geq 1-\delta$, 
using $\cO(d\ln(B/\varepsilon))$ oracle calls and $\tilde\cO(\frac{d^2 H^5 A B^4}{\varepsilon^4})$ interactions with the MDP. Furthermore this algorithm is computationally efficient.
\end{theorem}

 The pseudo-code for \Delphi is given in Algorithm \ref{alg:delphi}, which uses Algorithm \ref{alg:approx-measure} (\measuretd) as a sub-routine to measure expectations. 
 
\paragraph{Intuition for \Delphi} Recall that the expert policy $\pi^\circ$ satisfies $v^\circ = \cT^{\pi^\circ} v^\circ$, and that this fixed point is unique. We say that a candidate value function is \textit{consistent} at a state $s$ if $v(s) = \cT^{\pi^\circ}v(s)$. Note that we need consistency to hold at all states in order to ensure that $v = v^{\pi^\circ}$.

\Delphi is inspired by a recent algorithm of \citet{weisz2021tensorplan} called TensorPlan. As in TensorPlan, \Delphi proceeds via a ``guess and check'' procedure: at every iteration, we pick the \textit{optimistic} linear parameter which is \textit{consistent} on the past expert data that we have seen. %
Let's call the parameter chosen during a certain iteration $t$ as $\theta_t$. We then check whether this choice of parameters is globally consistent, by playing $\nroll$ rollouts of length $H$ with a policy derived from $\theta_t$. More specifically, the policy $\pi_{\theta_t}$ will play \textit{any action} $a$ such that $v_{\theta_t}$ is consistent with the Bellman update for that action, i.e. any action $a$ such that $v_{\theta_t}(s) = r(s,a) + \langle P(s,a),v_{\theta_t}(\cdot)\rangle$. (In reality, these expectations are estimated by playing the transition $(s,a)$ repeatedly, using the $\reset$ function.)

After a certain number of rollouts, one of two things happen: either this policy encounters a state where there is no consistent action, or we only encounter states that have a consistent action. In the first case, we are also inconsistent for the expert action at that state (since all actions are inconsistent), thus we query the oracle and update the parameter set. %
In the second case, we derive (cf. Lemma \ref{lem:virtual}) that if no inconsistencies are observed for several rollouts, then our ``virtual value'' $v_\theta$ is equal to the true value under $\pi_\theta$ (i.e., $v^{\pi_\theta}$). Using the optimistic property, this implies that we are optimal.

The only thing left to argue is that the number of iterations (i.e. the number of times that we can continue finding new parameters which are not globally consistent) is small. Using linearity of $v^\circ$, it turns out
roughly $d$ inconsistencies are sufficient for this. To see this, note that we can re-write the Bellman equation for any $v_\theta(\cdot) = \langle \phi(\cdot), \theta \rangle $ as:
\begin{equation}
    v_\theta(s) = \cT^{\pi^\circ}v_\theta(s)  \iff ~ 0 = r(s,\pi^\circ(s)) + \langle \bE[\phi(s')] - \phi(s), \theta \rangle \iff 0 = \langle  \Delta_{s,a} , 1 \oplus \theta \rangle, \label{eq:ortho}\\
\end{equation}
where we have used linearity of expectation, linearity of inner products, the definition of the direct product, and introduced the notation $\Delta_{s,a} \coloneqq r(s,a) \oplus \left(\bE[\phi(s')] - \phi(s)\right)$. We call the vector $\Delta_{s,a}$ the \textit{temporal difference} (TD) vector for $(s,a)$.
Equation \eqref{eq:ortho} is precisely an orthogonality constraint in $d+1$ dimensions. Thus, the parameter $\theta_t$ which is chosen at time $t$ is orthogonal to the previous $t-1$ TD vectors which have been generated from interactions with the oracle. 
If we happen to find a state which has no consistent action, then the TD vector corresponding to the expert action at that state must not be in the span of the previous expert TD vectors (otherwise it would be consistent). It follows that the iteration complexity is at most $d+1$, since there are at most $d+1$ linearly independent vectors in $\bR^{d+1}$. We use the Eluder dimension \citep{russo2013eluder} to generalize this argument to the case where the expectations are estimated. 

The next section solidifies the above intuition and sketches the proof more formally.

\begin{algorithm}[t]
\caption{\Delphi}\label{alg:delphi}
\begin{algorithmic}[1]
\State \textbf{Inputs:}
$s_0,\phi,\texttt{sub-optimality }\epstarget, \texttt{confidence } \delta, \texttt{parameter bound } B$
\State $\Theta_0 \gets \Ball_{\ell_2}(B)$ \Comment{$\Theta_t:$ current consistent parameters}
\State Initialize $E_d$, $\neval$,$\nroll$, and $\epstol$ via Equations \eqref{eq:eluder}, \eqref{eq:neval}, \eqref{eq:nroll}, \eqref{eq:epstol}
\For{$t=1$ to $E_d+1$}
	\State
	Pick $\theta_t \in\argmax_{\theta\in \Theta_{t-1}} \left( v_\theta(s_0) \coloneqq \theta^\top \phi(s_0)\right)$
		\Comment{Optimistic choice over $\Theta_{t-1}$} \label{line:new-theta} \label{line:new-iter}
	\State \texttt{consistent} $\gets$ \texttt{true}
	\For{$m=1$ to $\nroll$}			\Comment{$\nroll$ number of rollouts with $\theta_t$-induced policy} \label{line:nroll}
		\State  $S_{t,m,h}=s_0$ \Comment{Initialize rollout}
		\For{$h=1$ to $H$}
		    \For{$a \in [A]$} \Comment{For each action}
			\State
    			$\hat\Delta_{S_{t,m,h},a} \gets \measuretd(S_{t,m,h},a, \neval)$ \Comment{Measure TD vector at $(s,a)$} \label{line:avg-calc}
            \EndFor
			\If{
    	        $\min_a \left| \langle \hat{\Delta}_{S_{t,m,h},a}, 1 \oplus \theta_t \rangle \right| >\epstol$} \Comment{No consistent action} \label{line:consistency-test}
    	        \State \texttt{consistent} $\gets$ \texttt{false}
			    \State $a^\circ_t \gets \query(S_{t,m,h})$ \Comment{ Query oracle for $\pi^\circ(S_{t,m,h})$}
				\State $\tilde\Delta_{S_{t,m,h},a^\circ_t} \gets \measuretd(S_{t,m,h},a^\circ_t, 4E_d\neval)$ \label{line:avg-calc2} \Comment {Refined data}
				\State $\Theta_t \gets \Theta_{t-1} \cap \{\theta \mid |\langle \tilde\Delta_{S_{t,m,h},a^\circ_t}, 1 \oplus \theta \rangle | \leq \epstol \}$ \Comment{New admissible parameters} \label{line:new-params}
				\State Exit current iteration, $t \gets t+1$, Goto Line \ref{line:new-iter}.
			\EndIf
            \State $A_{t,m,h}\gets \argmin_{a\in[A]} \left|\langle \hat\Delta_{S_{t,m,h},a}, 1\oplus \theta_t \rangle \right|$ \label{line:action-choice} \Comment{Else consistent, keep playing}
			\State Play $A_{t,m,h}$, get $R_{t,m,h},S_{t,m,h+1} \sim \texttt{MDP}$
			 \Comment{Roll forward} \label{line:simulate-choice}
		\EndFor
	\EndFor
	\If{\texttt{consistent} == \texttt{true}}
	    \State \Return $\theta_{t}$ \Comment{No inconsistency for $m$ rollouts $\implies$ success} \label{line:cleantest-break}
	\EndIf
\EndFor
\State \Return $\theta_{E_d+1}$
\end{algorithmic}
\end{algorithm}

\begin{algorithm}[H]
\caption{\measuretd}\label{alg:approx-measure}
\begin{algorithmic}[1]
\State \textbf{Inputs:}
$s,a,\phi(\cdot),n,\reset()$
\For{$i=1$ to $n$}
  \State Play action $a$ at $s$, receive sample $R_l$ and $S'_l$ from MDP
  \State $\Delta_i\gets R_l \oplus \left(\phi(S'_l)-\phi(s)\right)$\label{line:approx-measure}
  \State $\reset()$
\EndFor
\State \Return $\hat\Delta_{s,a}:=\frac{1}{n}\sum_{i\in[n]} \Delta_i$

\end{algorithmic}
\end{algorithm}

\subsection{Proof sketch}
The full proof comes in 4 parts. The proofs for all Lemmas are provided in Appendix \ref{app:ub}.
\begin{enumerate}[leftmargin=*]
    \item Lemmas \ref{lem:conc} and \ref{lem:conc-star} gives concentration bounds which establish that, with high probability, the measurements $\hat\Delta$ (Line \ref{line:avg-calc}) and $\tilde\Delta$ (Line \ref{line:avg-calc2}) concentrate to the average $\Delta_{s,a} = \bE_{R(s,a)}r \oplus \left(\bE_{P(s,a)} \phi(s') - \phi(s) \right)$.
    \item Lemma \ref{lem:theta-star} establishes that, with high probability, the true optimal parameter $\theta^\circ$ is not eliminated from $\Theta_t$ for any parameter set that is encountered. It follows (Lemma \ref{lem:optimism}) by optimism that $v_{\theta_t}(s_0) \geq v^\circ(s_0)$ with high probability, where $\theta_t$ is the parameter chosen at time $t$.
    \item Lemma \ref{lem:iter} establishes an iteration bound: the algorithm will terminate after at most $t = E_d$ iterations of the outermost loop (and thus after at most $E_d$ oracle queries). The quantity $E_d$ happens to be the Eluder dimension of our linear function class.
    
    \item Lastly, Lemma \ref{lem:virtual} establishes that if $\nroll$ number of rollouts occur without observing a consistency break, then the virtual value ($v_\theta(s_0)$) must be roughly equal to the true value under the executed policy ($v^{\pi_\theta}(s_0)$). Theorem \ref{thm:delphi-ub} combines all the ingredients to conclude the proof.
\end{enumerate}

\subsubsection{Part 1: Concentration bounds}
Recall that $\Delta_{s,a} \coloneqq r(s,a) \oplus \left(\bE[\phi(s')] - \phi(s)\right)$ is the true TD vector, $\hat\Delta$ is the estimated TD vectors obtained with $\neval$ samples (in Line \ref{line:avg-calc}), and $\tilde\Delta$ is the ``refined data'' obtained with $4 E_d \neval$ samples (in Line \ref{line:avg-calc2}). The following lemmas establish concentration of $\hat\Delta$ and $\tilde\Delta$ to the true TD vector. 
\begin{restatable}[Concentration of $\hat\Delta_{s,a}$ (Line \ref{line:avg-calc})]{lemma}{conc}\label{lem:conc}
For any $s,a \in \cS \times \cA$ that is observed throughout the execution \Delphi, with $\neval$ samples in Line \ref{line:avg-calc}, we have that with probability  $\geq 1-\delta$, $\norm{\hat{\Delta}_{s,a}-\Delta_{s,a}}_\infty \leq \epseval$ and thus that $\langle 1 \oplus \theta , \hat{\Delta}_{s,a} - \Delta_{s,a} \rangle \leq \bepseval $.
\end{restatable}
\begin{restatable}[$\tilde\Delta_{s,a}$ concentrates even more (Line \ref{line:avg-calc2})]{lemma}{concstar}\label{lem:conc-star}
Similarly, for all $s,a$ where we call the oracle, with probability $1-\delta$, we have $ \norm{\tilde{\Delta}_{s,a} - \Delta_{s,a}}_\infty \leq \epseval / (2\sqrt{E_d})$, and thus, $\forall \theta \in \Ball_{\ell_2}(B)$, 
$
|\langle 1 \oplus \theta, \tilde{\Delta}_{s,a} - \Delta_{s,a} \rangle |\leq \bepseval/ (2\sqrt{E_d}).
$
\end{restatable}

\subsubsection{Part 2: Optimism}

This part shows that (with high probability) the true optimal parameter is not eliminated from the version space, and thus by optimism that the predicted value $v_t$ upper bounds $v^\circ$.

\begin{restatable}[$\theta^\circ$ not eliminated]{lemma}{thetastar}\label{lem:theta-star}
With probability~$\geq 1- \delta$, $\theta^\circ \in \Theta_t$ for all iterations $t \in [E_d+1]$.
\end{restatable}

\begin{restatable}[Optimism]{lemma}{optimism}\label{lem:optimism}
Under the event of Lemma \ref{lem:theta-star}, we have $v_t(s_0) \geq v^\circ(s_0), ~ \forall t \in [E_d]$.
\end{restatable}

\subsubsection{Part 3: Iteration bound}

To bound the iteration complexity of our algorithm, we use the notion of \textit{Eluder} dimension. Loosely, the Eluder dimension with respect to some target function is the longest sequence of points $(x_i)$ such that there exists functions differing from the target function on $x_i$ but which correctly fit it on $x_1,\dots,x_{i-1}$. A formal definition is provided in Appendix \ref{app:ub}. We will use the result that the Eluder dimension of linear functions is $\cO(d \ln(B/\varepsilon))$. 

\begin{restatable}[Iteration Complexity]{lemma}{iter}\label{lem:iter}
With probability $\geq 1- 2\delta$, the iteration complexity of the algorithm is at most the Eluder dimension at scale $\bepseval$, i.e. $E_d = \cO(d \ln (B/\bepseval))$.
\end{restatable}

\subsubsection{Part 4: Consistency, and putting everything together}

\begin{restatable}
[Consistency $\implies$ accurate prediction]{lemma}{virtual}\label{lem:virtual}
If $m$ rollouts have occured without any inconsistencies (i.e., the if statement of Line \ref{line:consistency-test} never gets triggered), then $v^{\pi_\theta}(s_0) > v_\theta(s_0) - 5H\bepseval - \epsroll$ with probability $\geq 1-3\delta$.
\end{restatable}

\begin{proof}[Proof (of Theorem \ref{thm:delphi-ub})]
Assume all events introduced so far (i.e. the events in Lemma \ref{lem:conc}, Lemma \ref{lem:conc-star}, and Lemma \ref{lem:virtual}). Together these happen with probability $\geq 1-3\delta$, so we can re-define $\delta \mapsto \delta/3$ such that the events happen together with probability $\geq 1- \delta$ (this only increases logarithmic factors by a factor of $3$). By Lemma \ref{lem:virtual}, we have:
\begin{align*}
v^{\pi_\theta}(s_0) &\geq v_\theta(s_0) -  5H\bepseval - \epsroll \\
&\geq v_\theta(s_0) - \epstarget \\
&\geq v^{\pi^\circ}(s_0) - \epstarget,
\end{align*}
where the second step follows from plugging in the definitions of $\neval$ (Eq. \eqref{eq:neval}), $\nroll$ (Eq. \eqref{eq:nroll}), $\bepseval,$ (Eq. \eqref{eq:epseval}) and $\epsroll$ (Eq. \eqref{eq:epsroll}), and the final step follows by optimism (Lemma \ref{lem:optimism}).
The total sample complexity of our algorithm is: $E_d = \tilde\cO(d)$ oracle calls, and $N \neval = (E_d+1) H \nroll A \neval = \tilde\cO(\frac{d^2 H^5 A B^4}{\varepsilon^4})$ exploration cost.
As for computational efficiency, we note that the only computationally intensive step is Line \ref{line:new-iter}, i.e. the optimization problem corresponding to the optimistic choice over the parameter set
$$
\max_{\theta \in \Theta_{t-1}} v_\theta(s_0) = \max_{\theta \in \Theta_{t-1}} \theta^\top \phi(s_0).
$$
This is readily seen to be a convex program, since the objective is a linear function and the constraint set is a convex set (cf. Line \ref{line:new-params}), and thus can be solved efficiently \citep{boyd2004convex}. %
\end{proof}

\subsection{Extensions}

In this section, we show that \Delphi can be extended to work with stochastic starting distributions, with misspecification, and with linear $q^\circ$ whenever dynamics are deterministic.

\paragraph{Stochastic start state}
We simply work with $v_\theta(\mu_0) = \bE_{s_0 \sim \mu_0}[v_\theta(s_0)]$ (resp. $v^\circ(\mu_0)$) wherever $v_\theta(s_0)$ (resp. $v^\circ(s_0)$) previously appeared. The ``starting feature'' $\bE_{s_0 \sim \mu_0}[\phi(s_0)]$ must be estimated from samples, which is then used for the optimistic program in Line \ref{line:new-iter} with $\phi(s_0)$ replaced by this expectation. The error is easily bounded as before by Hoeffding's inequality, and
will simply propagate additively through the proof.

\paragraph{Misspecified value functions and innacurate simulators} 

\Delphi inherits some robustness properties from TensorPlan. Namely, with a constant increase in exploration cost, \Delphi continues to work under errors in the modelling assumptions. The first case is where the expert value function is not linear but rather is approximately linear up to some uniform error. Formally, we say that the MDP is \textit{$\eta$-misspecified} for the expert policy $\pi^\circ$ and the feature map $\phi$ if there exists $\theta^\circ$ such that $\sup_s | v^\circ(s) - \langle  \phi(s), \theta^\circ\rangle | \leq \eta$. 
The second case is where the simulator itself is flawed. Formally, we say that the simulator is \textit{$\lambda$-innacurate} if a transition $(r,s')$ from any state-action pair $(s,a)$ of the MDP is instead observed as $(\Pi(r + \lambda_{s,a}), s')$, where $\Pi$ is the projection onto $[0,1]$ and $\lambda_{s,a}$ is a constant uniformly bounded by $\lambda$. The following result (proved in Appendix \ref{app:misspecification}) states that \Delphi can tolerate misspecification or simulator inaccuracies of order roughly $ \frac{\bepseval}{\sqrt{E_d}} = \cO(\frac{1}{H\sqrt{d}})$.

\begin{restatable}[\Delphi with misspecification]{theorem}{inaccurate}\label{thm:innacurate}
Redefine $\neval' = 4\neval$ (Eq. \ref{eq:neval}) and all subsequent hyperparameters. Then we have that, for all MDPs that are at most $\frac{\bepseval}{8\sqrt{E_d}}$-misspecified or for all simulators that are at most $\frac{\bepseval}{4\sqrt{E_d}}$-innaccurate, the conclusions of Theorem \ref{thm:delphi-ub} continue to hold when running \Delphi with the new hyperparameters.
\end{restatable}

\paragraph{$q^\circ$-linearity, in deterministic dynamics}

 Rather than working with the Bellman equation $v_\theta(s) = \cT^{\pi^\circ}v_\theta(s)$ we work with $q_\theta(s,a) = \cT^{\pi^\circ}q_\theta(s,a)$, which can be linearized similarly to Eq. \eqref{eq:ortho}. Namely:
\begin{equation}
    q_\theta(s,a) = \cT^{\pi^\circ}q_\theta(s,a)   \iff 0 = \langle  r(s,a) \oplus \left(\bE[\phi(s',\pi^\circ(s))] - \phi(s,a)\right), 1 \oplus \theta \rangle \label{eq:ortho-q} 
\end{equation}
This derivation holds generally, although to be able to speak of consistency at $(s,a)$ with respect to a next action $a'$, we now assume deterministic dynamics, so that the above becomes
$$
    0 = \langle  r(s,a) \oplus \left(\phi(s',a') - \phi(s,a)\right), 1 \oplus \theta \rangle,
$$
where $s'$ is the unique successor of $(s,a)$ and $a'$ is the action that we are checking consistency for. The algorithm proceeds as before, except rather than checking all actions at a given state (Line \ref{line:consistency-test}), we check all proceeding actions $a'$ against the current $(s,a)$, and then rollout the one with the smallest TD vector (in the sense of Eq. \ref{eq:ortho-q}).

 \section{How many oracle calls are necessary?}\label{sec:lb}

 Is \Delphi optimal in terms of its number of oracle calls? To answer this question, we must argue that no algorithm can find an $\varepsilon-$optimal solution with less than $\Omega(d)$ expert queries.
As stated, we are competing with agents which can (for example) exhaustively search the state-space and do not refer to the expert at all. Thus, an \textit{exploration budget} must be imposed (formally, a maximal amount of allowed interaction with the MDP, excluding oracle calls). To make matters more interesting, we set this cap to be \textit{any} polynomial amount, resulting in the following question: 
\begin{quote}
\textit{Under Assumptions \ref{ass:inter-exp}, \ref{ass:exp-lin}, and \ref{ass:reset}, 
what is the minimal amount of expert queries needed for any algorithm with a $\poly(d,H,A, \tfrac 1 \varepsilon)$ exploration budget to find an $\varepsilon-$optimal solution?}
\end{quote}
We note, firstly, that the minimal amount of expert queries is strictly positive in the worst-case, since there exists MDPs satisfying $v^\star$-linearity for which no algorithm can return a sound solution with $\poly(d,H,A, \tfrac 1 \varepsilon)$ queries \citep{weisz2021exponential,weisz2021tensorplan}. Secondly, while it was possible to restrict ourselves simply to learners which have the same exploration requirements as $\Delphi$, we opted to study algorithms with arbitrary polynomial exploration budgets, since it is a more fundamental question about the limits of exploration and the benefits of expert advice.

Note that any solution to this question must, a priori, have an exponential sample complexity for pure RL (otherwise the agent does not need to resort to the expert). Interestingly, most constructions which exhibit exponential lower bounds for linearly-realizable RL can be solved with a single query from the oracle (e.g., \citep{weisz2021exponential,wang2021exponential,foster2021statistical}). These constructions rely on having an exponentially large action set with a single correct action that effectively solves the MDP. 
Our main lower bound comes from extending the recent lower bound of \cite{weisz2021tensorplan}, which is also the only known construction for an exponential lower bound which has a polynomial action set (rather than exponential). Our result is that at least $\tilde\Omega(\sqrt{d})$ oracle calls are necessary: 
	
\begin{theorem}\label{thm:sqrt-d}
 There exists a family of MDPs satisfying \ref{ass:inter-exp}, \ref{ass:exp-lin}, and \ref{ass:reset}, such that any algorithm with $\poly(d,H,A)$ exploration budget will need at least $\tilde\Omega(\sqrt{d})$ oracle calls to recover a policy such that $v^{\hat{\pi}}(s_0) \geq v^\circ(s_0) - 0.01$.
\end{theorem}

Our second lower bound considers an alternative assumption, which instead posits that the expert \textit{policy} is linear. Formally:
\begin{assumption}[$\pi^\circ$-linearity, with bounded features]\label{ass:exp-lin-pi}
The policy $\pi^\circ$ \textit{of the expert} is linear with known features $\varphi: \cS \times \cA \rightarrow \bR^d$, i.e.
\begin{equation}\label{eq:lin-exp-pi}
    \pi^\circ(s) \in \argmax_a \langle \phi(s,a) , \theta^\circ \rangle, ~ \forall s \in \cS,
\end{equation}
for some unknown $\theta^\circ \in \bR^d \setminus \{0\}$.\footnote{We exclude $0$ since otherwise this would simply be the class of all policies.} We further assume that $\norm{\phi(s)}_2 \leq 1 \, \forall s$ and that $\norm{\theta^\circ}_2 \leq B$ for some known $B \in \bR^d$. 
\end{assumption}

When $\pi^\circ = \pi^\star$, it is easy to see that this assumption is more relaxed than the assumption that $q^\star$ is linearly-realizable. In general however, $q^\circ$-linearity does not imply that $\pi^\circ$ is linear (see Appendix \ref{app:pi-lin} for an example), although it does imply that the greedy policy derived from $q^\circ$ is linear. We give a lower bound for this case which matches the upper bound of \Delphi up to logarithmic factors.

\begin{theorem}\label{thm:lin-pi}
 There exists a family of MDPs satisfying assumptions \ref{ass:inter-exp}, \ref{ass:exp-lin-pi}, and \ref{ass:reset} such that any algorithm with $\poly(d,H,A)$ exploration budget will need at least $\Omega(d)$ oracle calls to recover a policy such that $v^{\hat{\pi}}(s_0) \geq v^\circ(s_0) - 0.01$.
\end{theorem}

\paragraph{Intuition for the lower bound} Theorems \ref{thm:sqrt-d} and \ref{thm:lin-pi} use the same MDP construction but with different features. We give some intuition for the MDP construction which is used, but due to its intricacy a full description (and the information-theoretic proof) are deferred to Appendix \ref{app:lb}.  Loosely, the learner has to find a hidden hypercube vector $s^\star \in \{\pm 1\}^p$. The action space is $\cA = [p]$, and each action corresponds to flipping one of the bits of the current vector. The MDP has $K$ ``phases'' which each correspond to $p$ bit flips (thus $H \approx Kp$). A linear reward is given only if a sufficiently small neighborhood of $s^\star$ is reached, and the reward (thus the value) will decay geometrically in each subsequent phase that the neighborhood is not reached. Intuitively, the oracle needs to be used $\approx p$ times, since each oracle calls only reveals one action, and thus one bit of the optimal vector $s^\star$. The reason that this results in a $\tilde\Omega(\sqrt{d})$ lower bound (rather than $\tilde\Omega(d)$) is that the value function will experience a \textit{scale transition} when going from states where $s^\star$ is reachable given the remaining steps in the current phase to states where $s^\star$ is no longer reachable. As just described, the value (and thus the features) will be one order of magnitude smaller in this latter portion of the state space. As this would betray the location of the secret parameter, the value function is instead augmented to be quadratic in $p$ (roughly, the product of the distance achieved at the end of this phase and that of the next phase), which thus requires that $p \approx \sqrt{d}$ in order to observe linearity. On the other hand, the lower bound for $\pi^\circ$-linearity (Theorem \ref{thm:lin-pi}) can remain linear in $d$, since the definition (Eq. \ref{eq:lin-exp-pi}) is \textit{scale-insensitive}.

Closing the gap between our upper bound of $\tilde\cO(d)$ and our lower bound of $\tilde\Omega(\sqrt{d})$ remains a challenging but interesting question. In finite horizons, with access to a generative model or a $\reset$ method, we suspect that the only mechanism for creating a hard MDP is by geometrically decaying the maximum possible value for each stage, such that at the final stage of the MDP the (random) reward becomes exponentially small in $H$ (that is the approach taken here and in \cite{weisz2021exponential,weisz2021tensorplan}).\footnote{In particular, an exponentially small gap is necessary, since backwards induction-type methods are possible when $q^\star$ is linearly-realizable and have sample complexities scaling with the inverse gap \cite{du2019good,du2020agnostic}.} If the geometric decaying happens in ``phases'' then this implies that $\Omega(1)$ of the value is located in the first ``phase''. The tension, then, is to have such a construction while i) hiding this large value, and ii) forcing the learner to rely on several oracle calls to find it. Prior constructions \cite{weisz2021exponential,wang2021exponential} have hid the large initial value by choosing an exponentially large action set, although as discussed above these examples are solved with a single query of the expert oracle. Extending the ``needle in a haystack'' to occur over multiple decisions (cf. the ``phases'' used above) leads to increased oracle requirements, although due to the scale transition phenomenon observed above it is far from clear how to do this with phases of length $\approx d$. The question thus boils down to: is geometric value reduction \textit{necessary} for exponential lower bounds in this setting, and if so can we avoid the scale transition problem?
The situation is likely to be different in the \textit{online RL} setting. For example, the construction of \cite{wang2021exponential} extends \cite{weisz2021exponential} to the online setting, and does not need to decay rewards but instead adds an $\Omega(1)$ probability of death at every transition. This mechanism evidently does not work when the agent has resets, and 
since it is not known whether \Delphi can be extended to the online setting we opted to keep the settings consistent between our upper bound and our lower bound.

 \section{Related works}\label{sec:prior}

The closest body of work to our setting is the field of \textit{interactive} IL. As in our setting, interactive IL considers the case where the learner has access to an expert oracle that can be queried adaptively. It differs from our setting, however, since traditionally in IL the learner does not observe reward information. We further differ from the IL setting since we consider value function approximation rather than general policy classes, and since we assume access to a $\reset$ function. Despite that many demonstrations of interactive IL occur in simulated domains \citep{ross2011reduction, ross2013interactive, ross2014reinforcement}, the benefits of this feature have not previously been studied. As a result, the \Delphi algorithm gives a lower oracle complexity than would have been obtained from using existing IL algorithms.  In terms of rates, it is shown in \citet{agarwal2019reinforcement} that Behaviour Cloning (for the passive case) and AggreVaTe \citep{ross2014reinforcement,sun2017deeply} (for the interactive case) have errors $\cO(\frac{1}{(1-\gamma)^2}\sqrt{ \frac{ \ln( |\Pi|/\delta)}{N}})$ for discounted MDPs, which roughly translates to an oracle complexity of $N = \cO(dH^4/\varepsilon^2)$ when using the standard reduction $H \mapsto (1-\gamma)^{-1}$ and taking $\Pi$ to be the policies of the discretized value function space.\footnote{Thus, by standard results on covering numbers, $|\Pi| \approx (B/\rho)^d$ where $\rho$ is the discretization radius\citep{vershynin2018high}.} This is in sharp contrast to our $\tilde\cO(d)$ oracle calls, which in independent of $H$ and logarithmic in $1/\varepsilon$, and demonstrates the improvement due to exploration with the help of value-function approximation. 
In terms of linear structure, the works of \citet{abbeel2004apprenticeship,syed2007game} assume a known transition function and unknown linear rewards, and derive sample complexities of $\cO\left(\tfrac{d H^2 \log(dH/\varepsilon)}{\varepsilon^2}\right)$ and $\cO(\tfrac{H^2 \log(d)}{\varepsilon^2})$ respectively, but the algorithms involve (tabular) planning in MDPs thus are not computationally efficient. Most relevant is the recent work of \citet{rajaraman2021value}, which, in the reward-free case, assumes that the expert policy is linear (i.e., our Assumption \ref{ass:exp-lin-pi}). An error rate of $\tilde\cO(dH / N)$ is shown for Behaviour Cloning in this case, though no lower bound is given.

On the technical side,  the \Delphi algorithm is inspired by a recent algorithm called TensorPlan \citep{weisz2021query}. TensorPlan works for pure RL under Assumptions \ref{ass:exp-lin} and \ref{ass:reset} but has a sample complexity scaling as $\poly( \left(\tfrac{dH}{\varepsilon} \right)^{A}, B)$ and is computationally intractable. Our extension of TensorPlan naturally incorporates the expert demonstrations, while simultaneously (1) having low oracle requirements, (2) addressing the exponential sample complexity of TensorPlan, and (3) rendering the algorithm computationally efficient. %
Our approach is based only on finding value functions which satisfy the Bellman equation. Bellman error minimization approaches have appeared in other works (e.g. \citet{jin2021bellman,zanette2020learning,chen2019information}), but have always required a restrictive ``Bellman closedness'' assumption.
As discussed, our lower bound construction is an extension of the recent remarkable lower bound of \citep{weisz2021tensorplan}, although several aspects of the construction have been modified to obtain better rates. In particular, we modified the reward/value functions, the feature mappings, and introduced an expert policy which differs from the optimal policy. For the proof, our setting is more complex as the learner has adaptive access to a second information source (the oracle), and a more sophisticated information-theoretic argument was needed to show that the oracle does not leak too much information. 

 \section{Conclusion}\label{sec:conc}

 We presented the \Delphi algorithm for RL with an interactive expert. We saw that, with $\tilde\cO(d)$ oracle calls, exponential improvements in sample complexity are possible for generative RL with linearly-realizable optimal value functions. Compared to prior works on learning with an interactive expert, we also saw that \Delphi's oracle requirements were smaller, and in fact are independent of the horizon of the MDP. It would be interesting and fruitful to resolve the gap between the oracle complexity required by \Delphi and the one obtained from our lower bound (either answer would be surprising to the authors). It would also be fruitful to study the case of linearly-realizable \textit{action-value} functions in stochastic MDPs, which would potentially enable our method to be extended to the online setting.

\begin{ack}
PA gratefully acknowledges funding from the Natural Sciences and Engineering Research Council of Canada (NSERC). Work done in part while PA was an intern at Amazon.
\end{ack}

\printbibliography

\newpage

\appendix

 \section{Proof of Theorem \ref{thm:delphi-ub}}\label{app:ub}

  \paragraph{Hyperparameters for the algorithm} We define
\begin{align}
    E_d &= 3d \frac{e}{e-1} \ln \{3 + 3\left(2B/\varepsilon \right)^2    \} + 1  \label{eq:eluder}\\
    \nroll &= \tfrac{2H^2(1+2B)^2\log(\frac{2(E_d+1)}{\delta})}{\epstarget^2} \label{eq:nroll}\\
    N &= (E_d+1)\nroll H  A \label{eq:N} \\
    \neval &= \tfrac{50H^2(1+B^2)(d+1)\log(\frac{2(d+1)N}{\delta})}{\epstarget^2} \label{eq:neval}\\
    \epseval &= \sqrt{\frac{\log(\frac{2(d+1)N}{\delta})}{2\neval} }  \label{eq:epseval} \\
    \bepseval &= \sqrt{1+B^2}\sqrt{d+1}\epseval \\
    \epstol &= 4 \bepseval \label{eq:epstol} \\
    \epsroll &= H(1+2B)\sqrt{\frac{\log(2(E_d+1)/\delta)}{2\nroll}} \label{eq:epsroll}
\end{align}

As we will see, $E_d$ is the upper bound on the number of parameters observed before global consistency holds, $\neval$ is the number of samples taken to estimate each TD vector, $\nroll$ is the number of rollouts executed for each parameter, $N$ is the maximum number of state-action pairs which will be seen by \Delphi, $\epseval$ is the error between the estimated TD vector and the true TD vector, and $\epsroll$ is the error of replacing the expected total rewards with the average over several rollouts.

\subsection{Part 1: Concentration Inequalities}
Recall that we write
$$
    \Delta_{s,a} = r(s,a) \oplus \left( \bE_{P(s,a)}[\phi(s')] - \phi(s) \right)
$$
for the true TD vector for state-action $(s,a)$. We write $\hat\Delta_{s,a}$ for any estimated TD vector resulting from Line \ref{line:avg-calc}, and $\tilde\Delta_{s,a}$ for any doubly-estimated TD vector resulting from Line \ref{line:avg-calc2}.

\conc*

\begin{proof}
Pick a single $s,a$, and omit the dependence on $s,a$ for cleanliness. Starting with the reward concentration: at every step we collect $\neval$ iid samples from $R(s_i,a_i)$. By Hoeffding, since rewards are bounded in $[0,1]$, the empirical average satisfies
$$
\left\lvert \frac{1}{n} \sum r_i - \bE[r] \right\vert \overset{\text{w.p. } 1-\delta/((d+1)N)}{\leq} \sqrt{ \frac{\log(2(d+1)N/\delta)}{2\neval}}.
$$
Onto the transition probabilities. Since the features are such that $\norm{\phi}_2 \leq 1$, we also have that $\norm{\phi}_\infty \leq 1$. Thus, each coordinate of the feature map is bounded by $1$. By Hoeffding, for each coordinate $i$ of the vector, we have that
$$
| \frac{1}{n}{\sum_{j=1}^n} \phi_i^{(j)} - \bE \phi_i^{(j)} | \overset{\text{w.p. } 1-\delta/((d+1)N)}{\leq} \sqrt{ \frac{\log(2(d+1)/\delta)}{2\neval}}
$$
Thus, with a union bound over these events for each coordinate $i$ and the event for the reward concentration, we have that w.p. $\geq 1-\delta/N$
$$
\norm{\hat{\Delta}- \Delta}_\infty \leq \sqrt{\frac{\log(\frac{2(d+1)N}{\delta})}{2\neval} } = \epseval.
$$
Finally, by a union bound over all $N$ state-action pairs observed by \Delphi, we have that for every $s,a$ encountered:
$$
\norm{\hat{\Delta}_{s,a}- \Delta_{s,a}}_\infty \leq \sqrt{\frac{\log(\frac{2(d+1)N}{\delta})}{2\neval} } = \epseval \, \text{ with probability } \geq 1-\delta.
$$
We further find that:
$\norm{\hat{\Delta}-\Delta}_2 \leq \sqrt{d+1}\epseval$ (since $\Delta$ and $\hat{\Delta}$ are $(d+1)$-dimensional). By Cauchy-Schwartz this gives, $\forall \theta \in \Ball_{\ell_2}(B)$:
$$
|\langle 1 \oplus \theta, \hat{\Delta} - \Delta \rangle| \leq \norm{1 \oplus \theta}_2 \norm{\hat{\Delta} - \Delta}_2 = \sqrt{1 + \norm{\theta}^2_2}\norm{\hat{\Delta} - \Delta}_2 = \sqrt{1+B^2}\sqrt{d+1}\epseval = \bepseval
$$
\end{proof}

\concstar*

\begin{proof}
Follows from the same proof as Lemma \ref{lem:conc}, just replace $\neval$ by $4E_d \neval $.
\end{proof}

\subsection{Part 2: Optimism}

\thetastar*

\begin{proof}
Let $\tilde{\Delta}_t$ be the TD vector which is added to $\Theta_t$ at time $t \in [E_d]$. Note that, by Lemma \ref{lem:conc-star}, for any given $t$, we have
\begin{equation}\label{eq:theta-eval}
| \langle 1 \oplus \theta^\circ, \tilde{\Delta}_t - \Delta_t \rangle | \leq \bepseval/ (2\sqrt{E_d}),
\end{equation}
for all $t$, w.p. $\geq 1-\delta$. The parameter set is defined as
$$
    \Theta_t = \{\theta : |\langle 1\oplus \theta, \tilde{\Delta}_i\rangle| \leq \frac{\bepseval}{2\sqrt{E_d}}\, \forall i \in [t] \}.
$$
Note that, for $\theta^\circ$, this is equivalent to requiring
$$
| \langle 1 \oplus \theta^\circ, \tilde{\Delta}_t - \Delta_t \rangle | \leq \frac{\bepseval}{2\sqrt{E_d}},
$$
since $\langle 1 \oplus \theta^\circ, \Delta_t \rangle = 0$ for all $\Delta_t$ (recall Equation \ref{eq:ortho}). This holds by Eq. \eqref{eq:theta-eval}.
\end{proof}

\optimism*

\begin{proof}[Proof (of Lemma \ref{lem:optimism})]
Under the event of Lemma \ref{lem:theta-star}, $\theta^\circ$ is never eliminated from $\Theta_t$. Thus, by the optimistic update rule, we have $\theta_{t+1} = \argmax_{\theta \in \Theta_t} \theta^\top \varphi(s_0)$, and since $\theta^\circ \in \Theta_t$, then
$$
    \theta_{t+1}^\top \phi(s_0) = v_{\theta_{t+1}}(s_0) \geq  (\theta^\circ)^\top \phi(s_0) = v_{\theta^\circ}(s_0) = v^\circ(s_0)
$$.
\end{proof}

\subsection{Part 3: Iteration Complexity}

To bound the iteration complexity of our algorithm, we will need to introduce the notion of \textit{Eluder} dimension. We use a simplified form introduced by \cite{li2021eluder}, although the first definition is due to \cite{russo2013eluder}.
\begin{definition}[Eluder dimension \citep{li2021eluder}]\label{def:eluder}
Let $\cF$ be a real-valued function class on domain $\cX$. Fix a \textit{reference function} $f^\star \in \cF$, and a \textit{scale} $\varepsilon$. The \textbf{Eluder dimension} of $\cF$ at a scale $\varepsilon$, w.r.t.~$f^\star$, is the length $\tau \in \bN$ of the \textit{longest} sequence of points $\left((x_1,f_1), \dots (x_\tau,f_\tau)\right)$ such that
\begin{equation}\label{eq:eluder-seq}
\forall i \in [\tau]: \quad |f_i(x_i) - f^\star(x_i)| > \varepsilon, \quad \text{and} \quad \sum_{j<i}(f_i(x_j) - f^\star(x_j))^2  \leq \varepsilon^2.
\end{equation}
An \textit{Eluder sequence} of length $\tau$ (with respect to $f^\star$) is any sequence $(x_i, f_i)_{i=1}^\tau$ which satisfies Eq \eqref{eq:eluder-seq} for each $i$.
\end{definition}
In other words, the Eluder dimension is the length of the longest sequence of points such that, for each $i$, we can find a new function $f_i$ which is large with respect to $f^\star$ on $x_i$ but correctly fits $f^\star$ on historical data $x_1,\dots, x_{j}$, for  $j < i$. We will use the folllowing bound for the Eluder dimension of linear functions with $d$-dimensional parameters.
\begin{lemma}[\cite{russo2013eluder}]\label{lem:linear-eluder}
For any $f^\star$, the Eluder dimension of the function class $\cF = \{ f_\theta(x) = \theta^\top x \}$, assuming that $\norm{\theta}_2 \leq B$ and $\norm{x}_2 \leq \gamma$, is
$$
\text{dim}_E(\cF, \varepsilon) \leq 3d \frac{e}{e-1} \ln \{3 + 3\left(2B/\varepsilon \right)^2    \} + 1 = \cO(d \ln( B/\varepsilon)).
$$
\end{lemma}
This gives us enough to prove our iteration bound -- we will show that the sequence of linear parameters chosen by our algorithm together with each new TD vector obtained from the oracle forms an Eluder sequence with respect to $\theta^\star$

\iter*
\begin{proof}[Proof (of Lemma \ref{lem:iter})]
Assume the events of Lemma \ref{lem:conc} and Lemma \ref{lem:conc-star}, which happen together with probability $\geq 1-2\delta$.
Our function class is $\cF = \{ f_\theta(x) = \langle 1 \oplus \theta, x \rangle \}$. This is a subset of all linear functions on $d+1$ dimensions, so by Lemma \ref{lem:linear-eluder},\footnote{(and using that $\cF \subseteq \cF' \implies dim_E(\cF,\varepsilon) \leq dim_E(\cF',\varepsilon)$)} it has Eluder dimension at least $\cO((d+1)\ln(B/\varepsilon)) = \cO(d \ln(B/\varepsilon))$. We pick $f^\star = f_{\theta^\circ}$.
We will show that $\forall i$, the sequence $(\tilde{\Delta}_i, f_{\theta_i})$ forms an Eluder sequence with respect to $f^\star$, from which it will follow that its length is bounded by $E_d$. We do this by induction. The base case is obvious. By the constraint set definition (Line \ref{line:new-params}), we have $|f_{\theta_i}(\tilde{\Delta}_j)| = |\langle1\oplus \theta_i, \tilde{\Delta}_j \rangle| \leq \frac{\bepseval}{2\sqrt{E_d}} \, \forall j < i$. Assuming the event of Lemma \ref{lem:theta-star}, we also have $|f_{\theta^\circ}(\tilde{\Delta}_j)| \leq \frac{\bepseval}{2\sqrt{E_d}}$, since $\theta^\circ \in \Theta_{i-1}$. Thus
$$
|f_{\theta_i}(\tilde{\Delta}_j) - f_{\theta^\circ}(\tilde{\Delta}_j)| \leq 2\frac{\bepseval}{2\sqrt{E_d}} \, \forall i \implies \sum_{j < i} (f_{\theta_i}(\tilde{\Delta}_j)- f_{\theta^\circ}(\tilde{\Delta}_j))^2 \leq \bepseval^2\frac{i-1}{E_d} \leq \bepseval^2,
$$
where in the last inequality we have used $i-1\leq E_d$ by the induction hypothesis. Thus, the second condition of the Eluder dimension is satisfied. For the first condition, we want to show that $|f_{\theta_i}(\tilde{\Delta}_i) - f_{\theta^\circ}(\tilde{\Delta}_i)| > \bepseval$. Note that, by Lemma \ref{lem:theta-star}, we have $| f_{\theta^\circ}(\tilde{\Delta}_i) | \leq \bepseval/(2\sqrt{E_d}) \leq \bepseval$. Recall from Line \ref{line:consistency-test} that $\left| \langle \hat{\Delta}_i, 1 \oplus \theta_i \rangle \right| = | f_{\theta_i}(\hat{\Delta}_i)| > \epstol$. Using concentration and linearity, we have $|f_{\theta_i}(\hat{\Delta}_i)| - |f_{\theta_i}(\tilde{\Delta}_i)| \leq 2\bepseval$. Thus, $|f_{\theta_i}(\tilde{\Delta}_i)| \geq \epstol - 2 \bepseval = 2 \bepseval$ since $\epstol = 4 \bepseval$.
Putting this together gives
$$
| f_{\theta_i}(\tilde{\Delta}_i) - f_{\theta^\circ}(\tilde{\Delta}_i) | \geq f_{\theta_i}(\tilde{\Delta}_i) - f_{\theta^\circ}(\tilde{\Delta}_i) \geq 2\bepseval - \bepseval = \bepseval.
$$
And we are done.
\end{proof}

\subsection{Part 4: Consistency, and Putting Everything Together}

\virtual*

\begin{proof}[Proof (of Lemma \ref{lem:virtual})]
Assume the event of Lemmas \ref{lem:conc} and Lemmas \ref{lem:conc-star}, which together happen with probability $\geq 1-2\delta$. The third event which we assume will be introduced shortly.

For cleanliness, let us write $\theta$ for the final parameter which observes $m$ rollouts without consistency break. Observe the following calculation:

\begin{align*}
v^{\pi_\theta}(s_0) &= \bE_{\pi_\theta}[ \sum_{j \in [H]} r_{S_j,A_j} ] \\
&= \bE\left[ \left\langle \left(\sum_{j \in [H]} r_{S_j,A_j}\right) \oplus \phi(s_{H+1}), 1 \oplus \theta \right\rangle  \right] \tag{$\phi(s_{H+1}) = 0$} \\
&= \bE\left[ \langle \phi(s_0), \theta \rangle + \sum_{j \in [H]} \langle r_{S_j,A_j} \oplus (\phi(S_{j+1}) - \phi(S_j) ) , 1 \oplus \theta \rangle \right] \tag{telescoping sum}\\
&= \langle \phi(s_0), \theta \rangle + \bE\left[\sum_{j \in [H]} \langle  r_{S_j,A_j} \oplus \left(P_{S_j,A_j}\phi(\cdot) - \phi(S_j)\right), 1 \oplus \theta \rangle\right],
\end{align*}
Now observe that after $\nroll$ number of rollouts, we have $\nroll$ unbiased estimates of the expected trajectories. Thus we can approximate $\bE\left[\sum_{j \in [H]} \langle  r_{S_j,A_j} \oplus \left(P_{S_j,A_j}\phi(\cdot) - \phi(S_j)\right), 1 \oplus \theta \rangle\right] \approx \frac{1}{\nroll} \sum_{i=1}^{\nroll} \bE[r_{S^i_j,A^i_j}] \oplus \left(P_{S^i_j,A^i_j}\phi(\cdot) - \phi(S^i_j)\right)$, where $S^i_j$ and $A^i_j$ are the states and actions in horizon $j \in [H]$ of rollout $i \in [\nroll]$. %
More precisely, using Hoeffding's and that $\langle  \bE[r_{S^i_j,A^i_j}] \oplus \left(P_{S^i_j,A^i_j}\phi(\cdot) - \phi(S_j)\right)] , 1 \oplus \theta \rangle \leq 1 + 2B$, we can get
$$
\left| \frac{1}{\nroll}\sum_{i \in [m]}\left(\sum_{j \in [H]} \langle  \Delta^i_j , 1 \oplus \theta \rangle \right) - \left(\bE \sum_{j \in [H]} \langle  \Delta^i_j, 1 \oplus \theta \rangle \right) \right| \leq H(1+2B)\sqrt{\frac{\log(2(E_d+1)/\delta)}{2m}} \coloneqq \epsroll,
$$
w.p. $\geq 1-\frac{\delta}{E_d+1}$, where we wrote $\Delta^i_j = \bE[r_{S^i_j,A^i_j}] \oplus \left(P_{S^i_j,A^i_j}\phi(\cdot) - \phi(S_j)\right)$. %
By a union bound this happens for all $t \in [E_d+1]$ with probability $\geq 1- \delta$.
Picking up where we left off:
\begin{align*}
v^{\pi_\theta}(s_0)&\geq \langle \phi(s_0), \theta \rangle + \frac{1}{\nroll}\sum_{i=1}^{\nroll} \sum_{j \in [H]} \langle  \bE[r_{S^i_j,A^i_j}] \oplus \left(P_{S^i_j,A^i_j}\phi(\cdot) - \phi(S_j)\right)], 1 \oplus \theta \rangle - \epsroll\\
&\geq \langle \phi(s_0), \theta \rangle + \frac{1}{m}\sum_{i=1}^m \sum_{j \in [H]} \langle \hat{\Delta}_{s^i_j,a^i_j}, 1 \oplus \theta \rangle - H\bepseval - \epsroll \tag{evaluation error} \label{eq:eval-conc} \\
&\geq v_\theta(s_0) - \frac{1}{m}\sum_{i=1}^m \sum_{j \in [H]} 4\bepseval - H\bepseval - \epsroll \tag{consistency holds} \label{eq:consistency-holds} \\
&= v_\theta(s_0) - 5H\bepseval - \epsroll
\end{align*}
and we are done. The second inequality follows from Lemma \ref{lem:conc}, which holds for all steps in all $m$ trajectories in all $E_d+1$ iterations. The third inequality holds since, for every $S^i_j$, there exists a consistent action, i.e. an action such that $| \langle \hat\Delta^i_j, 1 \oplus \theta \rangle | \leq \epstol = 4 \bepseval$.
\end{proof}

\subsection{Simulator inaccuracy and misspecification} \label{app:misspecification}

We start with the case of inaccurate simulators. Recall that we say that a simulator is \textit{$\lambda$-innacurate} if the samples obtained are of the form $(\Pi(r + \lambda_{s,a}), s')$, for any $(s,a)$ and for some constant $\lambda_{s,a}$ that satisfies $|\lambda_{s,a}| \leq \lambda$.

\inaccurate*

We can repeat the proof of \Delphi, and in fact the only difference will be in Part 1 of the proof (Lemmas \ref{lem:conc} and \ref{lem:conc-star}) will continue to hold, which also implies that the rest of the proof will continue to hold. 

\begin{lemma}[Concentration of $\hat\Delta_{s,a}$ (Line \ref{line:avg-calc})]\label{lem:conc-misspecified}
For any $s,a \in \cS \times \cA$ that is observed throughout the execution \Delphi, with $\neval$ samples in Line \ref{line:avg-calc}, we have that with probability  $\geq 1-\delta$, $\norm{\hat{\Delta}_{s,a}-\Delta_{s,a}}_\infty \leq \epseval$ and thus that $\langle 1 \oplus \theta , \hat{\Delta}_{s,a} - \Delta_{s,a} \rangle \leq \bepseval $.
\end{lemma}
\begin{lemma}[$\tilde\Delta_{s,a}$ concentrates even more (Line \ref{line:avg-calc2})]\label{lem:conc-star-misspecified}
Similarly, for all $s,a$ where we call the oracle, with probability $1-\delta$, we have $ \norm{\tilde{\Delta}_{s,a} - \Delta_{s,a}}_\infty \leq \epseval / (2\sqrt{E_d})$, and thus, $\forall \theta \in \Ball_{\ell_2}(B)$, 
$
|\langle 1 \oplus \theta, \tilde{\Delta}_{s,a} - \Delta_{s,a} \rangle |\leq \bepseval/ (2\sqrt{E_d}).
$
\end{lemma}
\begin{proof}[Proof (of Lemmas \ref{lem:conc-misspecified} and \ref{lem:conc-star-misspecified}]
	
We repeat the proof of Lemmas \ref{lem:conc} and \ref{lem:conc-star}. Note that $\hat{\Delta}_{s,a}$ is an average of i.i.d. random variables with mean $\Delta'_{s,a} = (r(s,a) + \lambda_{s,a}) \oplus P_{s,a}\phi - \phi(s)$. As before, by Hoeffdings, we have $|\langle 1 \oplus \theta, \hat{\Delta} - \Delta'\rangle| \leq \bepseval/2$, where the factor of $1/2$ is due to the definition of $\neval'$. This gives that 
$$
|\langle 1 \oplus \theta, \hat{\Delta} - \Delta\rangle| \leq |\langle 1 \oplus \theta, \hat{\Delta} - \Delta'\rangle| + |\langle 1 \oplus \theta, \Delta' - \Delta\rangle| \leq \bepseval/2 + \lambda \leq \bepseval/2 + \bepseval/4\sqrt{E_d} \leq \bepseval
$$
	The proof of the second part follows similarly, since we have 
	$$
|\langle 1 \oplus \theta, \tilde{\Delta} - \Delta\rangle| \leq |\langle 1 \oplus \theta, \tilde{\Delta} - \Delta'\rangle| + |\langle 1 \oplus \theta, \Delta' - \Delta\rangle| \leq \bepseval/4\sqrt{E_d} + \lambda \leq \bepseval/2\sqrt{E_d}
$$
\end{proof}

Next we handle the value misspecification case. Recall that the MDP is \textit{$\eta$-misspecified} for the expert policy $\pi^\circ$ and the features $\phi$ if there exists $\theta^\circ$ such that $\sup_s | v^\circ(s) - \langle  \phi(s), \theta^\circ\rangle | \leq \eta$. Here, we use a reduction argument to show that an $\eta$-misspecified MDP can be reduced to a $2\eta$-innacurate simulator of a realizable MDP. Namely, using the same reduction as Appendix D of \cite{weisz2021query}, we can construct an alternative MDP $\cM'$ such that $\cM'$ is realizable but is a $2\eta$-inaccurate simulator of $\cM$. The result then follows from the first part of the proof.

\newpage 

  \section{Proof of Section \ref{sec:lb}}\label{app:lb}
  \subsection{The MDP construction}

Theorems \ref{thm:sqrt-d} and \ref{thm:lin-pi} use the same MDP construction, which is inspired by the recent and remarkable lower bound of \citep{weisz2021tensorplan} (itself an extension of the lower bound of \citep{weisz2021exponential}). We give an overview of the MDP and its features, and describe the changes from the original construction of \cite{weisz2021tensorplan}. The state space consists of a hypercube in $p$ dimensions, $\cS = \{\pm 1\}^p$, for some $p \in \bN$, and assume for simplicity that $p$ is divisible by $4$. More specifically, the state space at stage $h$ also contains the history of vectors encountered so far $(s_1, \dots, s_h)$, which is uniquely defined as the transitions are deterministic. The dimension $p$ will end up being $p \approx \sqrt{d}$ for the value-based lower bound and $p \approx d$ for the policy-based lower bound. We write $\rho(\cdot,\cdot)$ for the \textit{Hamming} distance on $\cS$, recalling that is a bilinear function of its arguments, i.e. $\rho(s_1,s_2) = \frac{1}{2} (p - \langle s_1, s_2\rangle)$. The action set is $\cA = [p]$, and each action $a \in [p]$ will correspond to flipping the $a^\text{th}$ bit of the current state. 
Each trajectory of horizon $H$ has $K$ \textit{phases} ($K \in \bN$), and each phase consists of $p$ steps (thus, $H = K p$). We write $s_0, s_1, \dots, s_k, \dots s_K$ for the states reached at the end of each phase, and when we need to we will use $s_{k,i}$ for a state at step $i$ of phase $k$. The start state is $s_0 = \vec{1}$, the all-ones vector. 

There is a special ``goal state'' $s^\star$, and reward is given only if a) at the end of any phase, $\rho(s_k,s^\star) \leq p/4$, or b) the learner reaches horizon $H$, i.e. the end of phase $K$. The reward function decays geometrically, and is defined according to the sequence of states $s_0, \dots s_k$ obtained at the end of each phase. Letting $g(s_1,s_2) \coloneqq 1 - \rho(s_1,s_2)/p$ denote one minus the proportion of bits where two states differ, the reward function for reaching a $p/4$-neighbourhood of $s^\star$ at stage $k$ (condition a) just described) is deterministic and has value
\begin{equation}\label{eq:reward}
r_{w^\star}((s_i)_{i=1}^k) = \left( \prod_{i = 1}^k g(s_{i-1},s_i) \right) g(s_i, s^\star) 
\end{equation}
The reward function at stage $H$ (condition b) above) is always given and has the same expectation as Eq. \eqref{eq:reward}, but will be a Bernoulli random variable. 

Modulo some exceptions (to be described shortly), the transition function is deterministic and is defined via 
$
\cP(\tau(s,a) | s,a) = 1$, where $\tau(s,a) = (s_1, \dots, - s_a, \dots s_p)$ corresponds to the new vector obtained from flipping the bit at index $a \in [p]$. The exceptions to these transition dynamics are if 1) a state within a $p/4$-neighbourhood of $s^\star$ is reached, or 2) a move is repeated. In the first case, the MDP transitions to a game-over state $\bot$ (after which nothing else is possible and no reward is given). In the second case, if the move is repeated within the first $p/4$ steps of a phase then the MDP similarly transitions to $\bot$, and if the move is repeated in the second $3p/4$ turns then the current state becomes frozen until the end of the current phase (i.e. no further bit flips are allowed). This implies that each bit can only be flipped once in each phase, and further that $g(s_{i-1},s_i) \leq 3/4$ for each ``legal'' trajectory in the MDP (thus, the reward decays geometrically, cf. Equation \eqref{eq:reward}). 

So far, the main modification in our construction from that of \cite{weisz2021tensorplan} is the reward function. Their $g$ function is chosen to be $2^\text{nd}$ order in $\rho(s_1,s_2)$, while ours is linea in $\rho$. %
We now introduce the definition of the expert policy: we simply choose it to be the one which will flip the \textit{earliest} index such that $s_{k,i}$ differs from $s^\star$ and such that its index has not yet been played in the round. If no such index exists, or if $s_{k,i} = s^\star$, the expert policy will simply freeze the current round by repeating the earliest index that has already been played. (Note that there is always a repeated index in this case, and that repeating an index will lead to freezing instead of termination, since at the start of any phase $k$ the trajectory must satisfy $\rho(s_k,s^\star) > p/4$, otherwise the trajectory would have terminated). It will turn out that the expert's trajectory\textit{ from the start state} will be identical to that of the optimal policy (and, thus, will be equally difficult to compete with, in the sense of Eq. \eqref{eq:pac-expert}). While our expert policy might seem optimal, this need not be the case for arbitrary states since it might sometimes be more advantageous for a state in phase $k$ which can no longer reach the $p/4$-neighbourhood of $s^\star$ to simply aim at minimizing the inevitable factor of $g(s_{k-1},s_k)$ that it will incur. The following lemma (proved in Appendix \ref{app:more-mdp}) gives an expression for the value function corresponding to our expert policy, and shows that the value function satisfies Assumption \ref{ass:exp-lin}. 
\begin{lemma}\label{lem:v-is-lin}
Let $s_{k,i}$ be a state in round $i$ of phase $k$. Note that, from $s_{k,i}$, the state $s_{k+1}$ that $\pi^\circ$ will reach at the end of the current phase is deterministic. The value function of $\pi^\circ$ is 
$$
v^\circ(s_{k,i}) = \left( \prod_{k' \in [k]} g(s_{k'-1},s_{k'}) \right)g(s_k,s_{k+1})g(s_{k+1},s^\star),
$$
and furthermore this is linear with features $\phi_v$ of dimension $d = \Theta(p^2)$.
\end{lemma}

The next lemma gives an expression showing that this same expert policy is linear in a different set of features $\phi_\pi$, and thus that the expert policy satisfies Assumption \ref{ass:exp-lin-pi}.
\begin{lemma}\label{lem:pi-is-lin}
There exists a feature map $\phi_\pi$ and parameter $\theta_\pi$ of dimension $d = \Theta(p)$ such that
	$$
\pi^\circ(s_{k,i}) = \argmax_a \{\langle \phi_\pi(s,a), \theta_\pi \rangle\}
$$
\end{lemma}

This lemma is also proved in Appendix \ref{app:more-mdp}. 

The essence of our lower bound is that each oracle call will reveal one bit of the secret state $s^\star$, and thus without $\tilde\Omega(p)$ calls and learner will need exponentially-many exploratory samples to solve this MDP. Thus, for the value-based lower bound, an agent with a $d$-dimensional feature map must be given an MDP which has a $p = \Theta(\sqrt{d})$-dimensional state space. However, for the policy-based lower bound, we can give the agent a $p=\Theta(d)$-dimensional MDP. 

\paragraph{Information-theoretic lower bound}

Information-theoretically, a learning algorithm can solve this MDP only if they recover the secret state $s^\star$. We let $\MDP(p,K)$ denote an instance of the above MDP with dimension $p$ and parameters $K$. Following the approach of \citep{weisz2021tensorplan}, we prove the sample complexity hardness by reducing each $\MDP(p,K)$ to an instance of an abstract game called $\CubeGame(p,K)$. Details on $\CubeGame$, and the proof of the following theorem, are deferred to Appendix \ref{sec:expert-game}. The main result is the following:

\begin{theorem}\label{thm:reduction}
Any learner which solves $\MDP(p,K)$ can be used to solve $\CubeGame(p,K)$. Unless the number of oracle calls is $\Omega(\frac{p}{\log p})$, any learner which is $0.01$-optimal on $\CubeGame(p,K)$, with a sample complexity of $N$, will need
$$
N \geq 2^{\Omega(p \wedge K)}.
$$ 
\end{theorem}

Combined with the fact that a learner provided with $d$-dimensional features is given an MDP with parameter $p \approx \sqrt{d}$, this gives the result of Theorem \ref{thm:sqrt-d}.

\subsection{Proofs of Lemma \ref{lem:v-is-lin}, \ref{lem:pi-is-lin}}\label{app:more-mdp}

We need some more notation in order to linearize the value function. Let $s_{k,i} \neq \bot$ be a state in step $i$ of phase $k$. We define the variable $\ctflip_{k,i} = \rho(s_{k,0}, s_{k,i})$, which simply measures the number of components flipped so far in round $k$ of step $i$, and $\fix_{k,i} \in \{0,1\}^p$ which is a vector with $1$ at a component $j$ if said component is currently frozen (i.e. if it has been played or if the entire state has been frozen), and $0$ otherwise. Similarly, there are two scalars $\efix_{k,i}$ and $\enfix_{k,i}$ which simply count the number of components which disagree with $s^\star$ that are currently frozen (for $\efix_{k,i}$) or not frozen (for $\enfix_{k,i}$). Note that only $\efix_{k,i}$ and $\enfix_{k,i}$ depend on $s^\star$, and in fact we have:
\begin{align}
    \efix_{k,i} &= \frac{1}{2}(\langle \vec{1}, \fix_{k,i} \rangle - \langle \fix_{k,i} \cdot s_{k,i}, s^\star \rangle ) \\
    \enfix_{k,i} &= \frac{1}{2}(\langle \vec{1}, \neg \fix_{k,i} \rangle - \langle \neg \fix_{k,i} \cdot s_{k,i}, s^\star \rangle ),
\end{align}
where $\fix_{k,i} \cdot s_{k,i}$ denotes component-wise multiplication, i.e. $(\fix_{k,i} \cdot s_{k,i})_j = (\fix_{k,i})_j \cdot (s_{k,i})_j$. Now, we have that:
\begin{lemma}[Value of $v^\circ$]\label{lem:v-circ}
Let $s_{k,i}$ be a state in round $i$ of phase $k$. Let $s_{k+1}$ denote the state that $\pi^\circ$ will reach at the end of the current phase when starting from $s_{k,i}$ (and note that this choice is deterministic given $s_{k,i}$, and that it may not be in the $p/4$-neighbourhood of $s^\star$). Then we have:
$$
v^\circ(s_{k,i}) = \left( \prod_{k' \in [k]} g(s_{k'-1},s_{k'}) \right)g(s_k,s_{k+1})g(s_{k+1},s^\star),
$$
or, overloading notation and letting $g(x) = 1 - x/p$, we have
\begin{equation}\label{eq:v-circ-lin}
v^\circ(s_{k,i}) = \left( \prod_{k' \in [k]} g(s_{k'-1},s_{k'})\right)g(\ctflip_{k,i} + \enfix_{k,i})g(\efix_{k,i}).
\end{equation}
\end{lemma}
\begin{proof}
Identical to \citep[Lemma 4.9]{weisz2021tensorplan}
\end{proof}

\begin{lemma}
The value function $v^\circ$ is linear in features $\phi_v$ with dimension $d = \Theta(p^2)$.
\end{lemma}
\begin{proof}
Starting from Equation \ref{eq:v-circ-lin}, we observe that only $\efix_{k,i}$ and $\enfix_{k,i}$ depend on $s^\star$. The first term in parentheses in simply a scalar which multiplies the features. We thus calculate linear expressions for $x = \ctflip_{k,i} + \enfix_{k,i}$ and $y = \efix_{k,i}$. Starting with $y = \efix_{k,i}$, we have:
\begin{equation*}
    y = \frac{1}{2}\left(\langle \vec{1}, \fix_{k,i}\rangle - \langle \fix_{k,i} \cdot s_{k,i}, s^\star \rangle \right) = a + \langle b, s^\star \rangle ,
\end{equation*}
where $a = \frac{1}{2}\langle \vec{1}, \fix_{k,i} \rangle$ and $b = -\frac{1}{2}\fix_{k,i} \cdot  s_{k,i}$. Similarly:
\begin{equation*}
    x = \ctflip_{k,i} + \frac{1}{2}\left(\langle \vec{1}, \neg \fix_{k,i}\rangle - \langle \neg\fix_{k,i} \cdot s_{k,i}, w^\star \rangle \right) = c + \langle d, s^\star \rangle ,
\end{equation*}
where $c = \ctflip_{k,i} + \frac{1}{2}(\langle \vec{1}, \neg \fix_{k,i}\rangle)$ and $d = -\frac{1}{2}\neg\fix_{k,i} \cdot s_{k,i}$. Thus we have: $g(y) = 1-y/p = 1- (a + \langle d,s^\star\rangle)/p = (1-a/p) - \langle \bar{b}, \bar{s^\star}\rangle = a' + \langle\bar{b},\bar{s^\star}\rangle,$ where $a' = (1-a/p)$, $\bar{b} = -b/\sqrt{p}$ and $\bar{s^\star} = s^\star/\sqrt{p}$. Similarly, $g(x) = 1 - x/p = 1 - (c + \langle d,s^\star\rangle)/p = c' - \langle \bar{d}, \bar{s^\star}\rangle,$ where $c' = 1-c/p$ and $\bar{d} = d/\sqrt{p}$. Putting this together we have that 
\begin{align*}
v^\circ(s_{k,i})&= \left( \prod_{k' \in [k]} g(s_{k'-1},s_{k'})\right)g(x) g(y) \\
&=  \left( \prod_{k' \in [k]} g(s_{k'-1},s_{k'})\right) (a'+\langle \bar{b}, \bar{s}^\star\rangle) (c'+\langle \bar{d}, \bar{s}^\star\rangle) \\
&= \left( \prod_{k' \in [k]} g(s_{k'-1},s_{k'})\right)\left( a'c'+c'\langle \bar{b}, \bar{s}^\star\rangle + a' \langle \bar{d}, \bar{s}^\star\rangle + \langle \bar{b}, \bar{s}^\star\rangle \langle \bar{d}, \bar{s}^\star\rangle \right) \\
&= \left( \prod_{k' \in [k]} g(s_{k'-1},s_{k'})\right)\left( a'c'+\langle c'\bar{b} + a' \bar{d}, \bar{s}^\star\rangle + \langle \bar{b} \otimes \bar{d}, \bar{s}^\star \otimes \bar{s}^\star \rangle \right),
\end{align*}
where we have use a property of the tensor product that $\langle a_1,b_1\rangle \langle a_2,b_2 \rangle = \langle a_1 \otimes a_2, b_1 \otimes b_2 \rangle$, where $a_1 \otimes a_2, b_1 \otimes b_2 \in \bR^{p \times p}$ and their inner product is interpreted as the inner product between the vectorized matrices. Thus, if we take 
$$
\theta_v = 1 \oplus \bar{s}^\star \oplus (\bar{s}^\star \otimes \bar{s}^\star) \in \bR^{1 + p + p^2}
$$
and 
$$
\phi(s_{k,i}) = \left( \prod_{k' \in [k]} g(s_{k'-1},s_{k'})\right)\left( a'c' \oplus \left(c'\bar{b} + a' \bar{d}\right) \oplus \left(\bar{b} \otimes \bar{d}\right) \right) \in \bR^{1+p+p^2},
$$
then we have $v^\circ(s_{k,i}) = \langle \phi(s_{k,i}),\theta_v \rangle$ as desired. Thus $v^\circ$ is linear with features in dimension $1 + p + p^2$. Note that the norm of the features and the parameter $\theta_v$ is also bounded by constants. 
\end{proof}

This completes the proof for $v^\circ$-linearity. We next tackle the analogous statement for $\pi^\circ$-linearity. 

\begin{lemma}
There exists a feature map $\phi_\pi$ and parameter $\theta_\pi$ of dimension $d \approx p$ such that
	$$
\pi^\circ(s_{k,i}) = \argmax_a \{\langle \phi_\pi(s,a), \theta_\pi \rangle\}
$$
\end{lemma}

\begin{proof}
Let $s_{k,i}$ be a state of interest. Recall that $\pi^\circ$ will either 1) flip the earliest index which has not been fixed such that the value of $s_{k,i}$ at that index disagrees with $s^\star$, or 2) if no such index exists, freeze the current round by playing a frozen index.
First consider a state $s_{k,i}$ such that $\pi^\circ(s_{k,i})$ will flip an index. Since the index flipped was previously incorrect and will thereafter agree with $s^\star$ on that bit, this corresponds to minimizing the distance between $s'$ and $s^\star$ amongst all possible $s'$ which can be reached in one step from $s_{k,i}$ (i.e. amongst all possible $\{\tau(s_{k,i},a)\}_a$, recalling that $\tau(s_{k,i},a)$ is our notation for the transition function of the MDP). Thus, $\pi^\circ(s_{k,i}) \in \argmin_a\{ \rho(\tau(s_{k,i},a), s^\star) \}$. This can be written linearly as $\argmin_a \{\frac{1}{2}(p - \langle \tau(s_{k,i},a), s^\star \rangle)  \} = \argmax_a \{ \frac{1}{2}(\langle \tau(s_{k,i},a), s^\star \rangle - p)\} = \argmax_a \{ \frac{1}{2}(\langle p \oplus \langle \tau(s_{k,i},a), s^\star \rangle , 1 \oplus s^\star \rangle) \}$. %
The second case is that $\pi^\circ$ will freeze the round at the state $s_{k,i}$. This means that there are no indices which are incorrect that have not been frozen in this round. Again, this corresponds to minimizing the distance between $\tau(s_{k,i},a)$ and $s^\star$: all other choices will either result in a game over (which has $0$ value) or will flip an incorrect bit (which increases the distance). Thus, again we have $\pi^\circ(s_{k,i}) \in \argmin_a\{ \rho(\tau(s_{k,i},a), s^\star) \} = \argmax_a \{ \frac{1}{2}(\langle p \oplus \langle \tau(s_{k,i},a), s^\star \rangle , 1 \oplus s^\star \rangle) \}$.

Thus, in either case, we have that $\pi^\circ$ is linear with features $\phi_\pi(s,a) = p \oplus \tau(s_{k,i},a)$ and $\theta^\circ = 1 \oplus s^\star$. Since the definition of the $\argmax$ is scale-insensitive, we can further normalize to obtain that the features are bounded in magnitude by a constant. 
\end{proof}

\subsection{$\CubeGame$ with expert advice, and Proof of Theorem \ref{thm:reduction}}\label{sec:expert-game}

Following the approach of \citep{weisz2021tensorplan}, we give our lower bound by providing a reduction to an abstract game called $\CubeGame$. Any learning algorithm which can solve the MDPs in our construction can also be used to solve $\CubeGame$, and thus it follows that the learner will be subject to the same lower bound. For our setting, we modify the reward function $\CubeGame$ and augment the learner with the ability to query an expert, which will behave identically to the expert policy which we have defined in our MDPs (that is, it will flip the first bit which is incorrect or give a special actions to indicate if all of the bits are correct). For the rest of this section, when referring to $\CubeGame$ we are referring to our modified version. 

In \citep{weisz2021tensorplan} it is shown that any algorithm which outputs a $0.01-$optimal answer for the (expert-less) $\CubeGame$ will need a query complexity of $N \geq 2^{\Omega(p \wedge K)}$. In what follows, we will provide the analogous proof of this for our modified game. Thus, the main result of this section is the following, which states that if the learner is not given a budget of $\Omega(p/ \log p)$ expert queries, the sample complexity remains exponential.  A learning algorithm for $\CubeGame$ will be called a \textit{planner}, and a planner which returns a $0.01$-optimal answer at the end of $\CubeGame$ will be called \textit{sound}.

\paragraph{Rules of $\CubeGame$} $\CubeGame$ is a bandit-like game which is defined by two parameters: a length $K \in \bN_+$ and a dimension $p \in \bN_+$. The ``action space'' is $W = \{\pm 1\}^p$. Recall that $\rho(x,y) = \frac{1}{2}(p - x^\top y)$ is the Hamming distance between two vectors in $W$. The secret parameter which solves the game is housed in the set $W^\star = \{w \in W \mid \tfrac{p}{4} \leq \rho(\vec{1}, w) \leq \tfrac{3p}{4}\}$. The planner can only input sequences of vectors where each vector is sufficiently far from the previous one. Formally, for any $k \in [K]$, we let $W^{\circ k} = \{(w_i)_{i \in [k]} \in W^k \mid \forall i \in [k]: \rho(w_{i-1},w_i) \geq p/4 \}$, with $w_0 \coloneqq \vec{1}$. The action space is:
$
\cA = \cup_{k \in [K]} W^{\circ k},
$
thus the planner can input any sequence of length $k \leq K$ satisfying that $(w_i) \in W^{\circ k}$. 

The reward function is defined by 
\begin{align*}
f_{w^\star}\left( (w_i)_{i \in [k]} \right) &= \left( \prod_{i \in [k]} g(\rho(w_{i-1},w_i)) \right) g(\rho(w_k,w^\star)),  \\
\text{ where: }g(x) &= 1 - \frac x p,
\end{align*}
with base case $f( () ) = g(\rho(w_0,w^\star))$.

While the planner plays, it chooses sequence lengths $L_t \in [K]$ and input sequences $S_t = (w_i^t)_{i \in [L_t]} \in W^{\circ L_t}$. If it chooses to stop playing, it chooses an output $S_t = (w_i^t)_{i \in [8]} \in W^{\circ 8}$ (we distinguish this case by letting $L_t = 0$ denote that the planner has chosen to terminate). Thus the number of actions taken is $N = \min\{t \in \bN_+ \mid L_t = 0\}$.

After any action $S_t$, the reward is given if either $\rho(w_{L_t}^t,w^\star) < p/4$ or if $L_t = K$. In both cases the reward is $\Ber(f_{w^\star}(S_t))$, a Bernoulli random variable with mean $f_{w^\star}(\cdot)$. Similarly, if the planner is done (i.e. the input is $S_N \in W^{\circ 8}$), then the reward given is $R = f_{w^\star}((w_i^N)_{i \in [k^\star]})$, where $k^\star = \min\{8, \min\{k \mid \rho(w_k^N, w^\star) < p/4\}\}$.

The last thing to specify is the oracle. Here, we allow the planner to query the oracle \textit{part-way} through a sequence. Namely, if the planner chooses to input a sequence of length $L_{t} < K$, then the oracle can be queried at the end of the sequence, and a second sequence of length $L^2_{t} \leq K - L_t$ can be inputted. This can be repeated as many times as desired, given that the total sequence length remains $\leq K$. The oracle will simulate the expert policy from before: upon being called at a vector $w_{k,i}$, it will either return the index of the first bit which does not agree with $w^\star$, or it will return a special action indicating that all bits are correct. 

We are now ready for the main theorem. 
\begin{theorem}\label{thm:cube-game}
Any sound planner for $\CubeGame$ which has less than $\Omega(p/\log p)$ oracle queries, will have a sample complexity
$$
N \geq 2^{\Omega(p \wedge K)}
$$ 
\end{theorem}

The proof comes in 5 lemmas, two of which (Lemmas \ref{lem:prop-f-star} and \ref{lem:hypercube-counting}) are analogous to properties from the expert-less $\CubeGame$. The other 3 lemmas are information-theoretic and are specific to the oracle setting. 

First, some properties about the reward function of $\CubeGame$.
\begin{lemma}[Properties of $f_{w^\star}$]\label{lem:prop-f-star}
For any $w^\star \in W^\star, k \in \bN$, $s = (w_{k'})_{k' \in [k]} \in W^{\circ k}$, we have
\begin{align*}
    \frac{1}{4} &\leq f_{w^\star}(()) \leq \frac{3}{4}\\
    0 &< f_{w^\star}(s) \leq \left(\frac{3}{4}\right)^{k + \mathbbm{1}[\rho(w_k, w^\star) \geq p/4]} 
\end{align*}
\end{lemma}
\begin{proof}
The proof is analogous to \citep[Lemma 4.2]{weisz2021tensorplan}, substituting our first-order $g$ function.
\end{proof}

The following parameters will control our sample complexity:
\begin{align}
n &= \min\left\{ \exp(p/8)p^{-x}/20 - 5, \left(\tfrac 1 \varepsilon - 1 \right) / 9.9 \right\}, \,\quad \text{ where } \label{eq:sample-comp} \\
x &= \frac{p}{16} \log p, \,\quad \text{ and } \label{eq:oracle-comp}\\
\varepsilon &= \left(\frac 3 4 \right)^{K+1} \label{eq:epsilon}
\end{align}
We will see that $N = \Omega(n)$ for any sound planner. 

In the original game of \cite{weisz2021tensorplan}, the interaction protocol is captured by $(X_t, Y_t)_{t \in [N]}$, where 
\begin{itemize}
    \item $N = \min\{t \in \bN_+ \mid L_t = 0\}$ is the interaction length
    \item $L_t \in [K]$ is the input length chosen,
    \item $S_t$ is the sequence $(w_i)_{i \in [L_t]}$ inputted, satisfying $\rho(w_{i-1},w_i) \geq \tfrac p 4$
    \item $X_t = (L_t, S_t)$,
    \item $U_t = \mathbbm{1}\{\rho(w_{L_t-1}, w^\star) < p/4 \}$,
    \item $V_t = \mathbbm{1}\{\rho(w_{L_t}, w^\star) < p/4 \}$, and
    \item $Z_t = 0$ unless $V_t = 1$ or $L_t = K$, in which case $Z_t = \text{Ber}(f_{w^\star}(S_t))$,
    \item $Y_t = (U_t, V_t, Z_t)$.
\end{itemize}

In our case, we need the sequence $S_t$ to include every bit flip, thus $S_t = (w_{k,i})_{k \in [L_t], i \in [p]}$%
We also have two new variables, namely $O_t = (o_{k,i})$ which is an indicator that the oracle was called at step $i$ of phrase $k$ and $E_t = (e_{k,i})$ which is the answer returned by the oracle. Thus our new interaction protocol is defined by $X_t = (L_t, S_t, O_t, E_t)$ and $Y_t = (U_t, V_t, Z_t)$, where $Y_t$ remains unchanged. 

The planner $A$, with $n$ interactions, in the environment defined by $w^\star$, defines a distribution over the environment: 
\begin{equation}\label{eq:factorize}
P^{A,n}_{w^\star}((X_t,O_t,E_t,Y_t)_t) = \prod_{i=1}^n p(x_i|x_{1:i-1},o_{1:i-1},e_{1:i-1},y_{1:i-1})p(o_i|x_{1:i},o_{1:i-1},e_{1:i-1},y_{1:i-1})p(e_i|x_i,o_i)p(y_i|x_i).
\end{equation}
Note that $p(x_i|x_{1:i-1},o_{1:i-1},e_{1:i-1},y_{1:i-1})$ and $p(o_i|x_{1:i},o_{1:i-1},e_{1:i-1},y_{1:i-1})$ are decisions made by the planner and $p(e_i|x_i,o_i)p(y_i|x_i)$ and $p(y_i|x_i)$ are obtained by querying the environment. We define the ``abstract game $(0,w^\star)$'' to always yield reward 0, and which has the same oracle as environment $w^\star$. It's distribution will be written as $P^A_{(0,w^\star)}$. Let $E_n^{w^\star}$ be the event that in the first $n$ steps the planner does not hit on any vector that is close to $w^\star$:
$$
E_n^{w^\star} = \cap_{t \in [n]} \left\{ t > N \text{ or } (t=N \text{ and } \min_{i \in [8]} \rho(w_i^N, w^\star) \geq \tfrac p 4) \text { or } (t < N \text{ and } \rho(w_{L_t-1}, w^\star) \geq \tfrac p 4 \text{ and } \rho(w_{L_t}, w^\star) \geq \tfrac p 4 )\right\}
$$
\begin{lemma}[A first change of measure]\label{lem:change-of-measure}
For any planner $A$ and any $w^\star \in W$, we have
$$
P^{A}_{w^\star}(E_n^{w^\star}) \geq \tfrac{9}{10} P^{A}_{(0,w^\star)}(E_n^{w^\star})
$$
\end{lemma}
\begin{proof}
It will be shown that 
\begin{equation}\label{eq:1-eps}
    P^{A}_{w^\star}(E_n^{w^\star}) \geq (1-\varepsilon)^n P^{A}_{(0,w^\star)}(E_n^{w^\star}).
\end{equation}
Since $1-\varepsilon \geq 1 - \frac{1}{1+9.9n}$ by our definition of $n$ we have that $(1-\varepsilon)^n \geq (1-\frac{1}{1+9.9n})^n \geq \lim_{n \rightarrow \infty} (1 - \frac{1}{1+9.9.n})^n > 9/10$, and thus Equation \eqref{eq:1-eps} implies our result. 

Let $\cH$ be the set of all possible histories $(x_t,o_t,e_t,y_t)$ of length $n$, and $E_h = E_n^{w^\star} \cap \{H=h\}$. Note that $E_n^{w^\star}$ is the disjoint union of $E_h$, so it is enough to show that for each $h$ we have:
$$
\rho = \frac{P_{w^\star}^{A}[E_h]}{P_{(0,w^\star)}^{A}[E_h]} \geq (1-\varepsilon)^n,
$$
for each $h$ such that $P_{(0,w^\star)}^A(E_h) > 0$. 
So, let $h = (x_t, o_t, e_t, y_t)$ be such that $P^A_{(0,w^\star)}(E_h) > 0$. This implies in particular that $y_t = (0,0,0) \, \forall t$. 
Now, both $P^A_{w^\star}$ and $P^A_{(0,w^\star)}$ factorize according to Eq. \eqref{eq:factorize}, giving:
\begin{equation*}
    \frac{P^A_{w^\star}[E_h]}{P^A_{(0,w^\star)}[E_h]} = \prod_i \frac{p(x_i|x_{1:i-1},o_{1:i-1},e_{1:i-1},y_{1:i-1})p(o_i|x_{1:i},o_{1:i-1},e_{1:i-1},y_{1:i-1})p^{w^\star}(e_i|x_i,o_i)p^{w^\star}(y_i|x_i)}{p^A(x_i|x_{1:i-1},o_{1:i-1},e_{1:i-1},y_{1:i-1})p^A(o_i|x_{1:i},o_{1:i-1},e_{1:i-1},y_{1:i-1})p^{(0,w^\star)}(e_i|x_i,o_i)p^{(0,w^\star)}(y_i|x_i)}.
\end{equation*}
Since we are conditioning on the same fixed history, the terms involving decisions made by the planner will cancel, and similarly the variable $E_i$ also behaves the same in both environments (since the oracle for $w^\star$ is the same). We are left with:
\begin{equation*}
    \rho = \prod_i \frac{p^{w^\star}(y_i=(0,0,0)|x_i)}{p^{(0,w^\star)}(y_i=(0,0,0)|x_i)} 
\end{equation*}
The denominator always has probability $1$ in the environment $(0,w^\star)$, and since $U_t = V_t = 0$ under the set $E_n^{w^\star}$ (the planner is never close to $w^\star$), we have $Y_t = 0 \iff Z_t = 0$, so it remains to control 
$$
\rho = \prod_{i=1}^n P^A_{w^\star}(Z_t = 0 | x_t).
$$
Again since the planner is never close to $w^\star$, $Z_t = 1$ only if $l_t = K$, in which case we have $P^A_{w^\star}(Z_t =1 | x_t) \geq (3/4)^{K+1} = \varepsilon$ by definition of the reward $f_{w^\star}$ obtained from reaching level $K$.
\end{proof}
The next lemma simply bounds the number of vectors in $W$ which are close to any fixed vector. 
\begin{lemma}[Hypercube counting]\label{lem:hypercube-counting}
For $\tilde w \in W$, let $\Wclose(\tilde w) = \{w \in W \mid \rho(w,\tilde w) < p / 4 \}$. Then $|\Wclose(\tilde w)| \leq 2^p \exp(-p/8)$
\end{lemma}
\begin{proof}
Identical to \citep[Lemma 4.4]{weisz2021tensorplan}.
\end{proof}

Recall that $E_n^{w^\star}$ is the ``bad event'' for the planner. We study its complement, $(E_n^{w^\star})^c$, which satisfies $(E_n^{w^\star})^c \subset \{ w^\star \in Z\},$ where
$$
Z \coloneqq \cup_{t \in [n \wedge (N-1)]} \left(\Wclose(w^t_{L_t-1}) \cup \Wclose(w^t_{L_t}\right) \cup ( \cup_{i \in [8]} \Wclose(w_i^N)),
$$
i.e. the event that for some $t$ we have $\rho(w^t_{L_t-1},w^\star) < p/4$ or $\rho(w^t_{L_t},w^\star) < p/4$ or that for some $i \in [8]$ we have $\rho(w_i^N, w^\star) < p/4$.
We define the ``abstract game'' $(0,0)$, where the planner has access to an oracle but, when queried, rather than giving information about the ``true'' $w^\star$, the oracle will simply return a uniformly random bit in $[p]$.

\begin{lemma}[A second change of measure]\label{lem:unif-oracle}
 For any planner $A$ with an oracle budget of $x$, we have that:
$$
P^A_{0,\hat w}(\hat w \in Z) = p^x P^A_{0,0}(\hat w \in Z)
$$
\end{lemma}
\begin{proof}
As before, consider the set of histories $\cH$, and let $Z_h = Z \cap \{H = h\}$. Writing out the importance ratio gives:
\begin{equation*}
    \rho = \frac{P^A_{0,w^\star}[Z_h]}{P^A_{(0,0)}[Z_h]} = \prod_i \frac{p(x_i|x_{1:i-1},o_{1:i-1},e_{1:i-1},y_{1:i-1})p(o_i|x_{1:i},o_{1:i-1},e_{1:i-1},y_{1:i-1})p^{0,w^\star}(e_i|x_i,o_i)p^{(0,w^\star)}(y_i|x_i)}{p(x_i|x_{1:i-1},o_{1:i-1},e_{1:i-1},y_{1:i-1})p(o_i|x_{1:i},o_{1:i-1},e_{1:i-1},y_{1:i-1})p^{(0,0)}(e_i|x_i,o_i)p^{(0,0)}(y_i|x_i)}
\end{equation*}
Again, as before, all the terms involving the planner will cancel, since they are conditioned on the same history and thus the planner will make the same decisions. Similarly, in both games the reward is deterministically $0$ thus the $p^{(0,w^\star)}(y_i|x_i) = p^{(0,0)}(y_i|x_i)$. We are left with
$$
\rho = \prod_{i=1}^n \frac{ p^{(0,w^\star)}(E_i = e_i \mid x_i,o_i)}{ p^{(0,0)}(E_i = e_i \mid x_i,o_i)}
$$
Note that the top probability is deterministic (since the true expert is) and is only equal to $1$ at most $x$ times (recalling that $x$ is the total number of oracle calls allowed). We are simply left with the (inverse of the) probability that the random oracle returns any given answer, which is $1/p$. Thus we end up with $\rho = p^x$.
\end{proof}

\begin{lemma}[Finding a bad $w^\star$ for planner $A$]
For any abstract planner $A$ there exists $w^\star \in W^\star$ such that 
$$
P_{w^\star}^A(E_n^{w^\star}) \geq (9/10)^2
$$
\end{lemma}
\begin{proof}
Note that, by Lemma \ref{lem:change-of-measure}, it is sufficient to show that $P^A_{0,w^\star}((E_n^{w^\star})^c) \leq \frac{1}{10}$. Recall that for any $\hat{w} \in W^\star$ we have $(E_n^{\hat w})^c \subseteq \{\hat w \in Z\}, $ where 
$$
Z \coloneqq \cup_{t \in [n \wedge (N-1)]} \left(\Wclose(w^t_{L_t-1}) \cup \Wclose(w^t_{L_t}\right) \cup ( \cup_{i \in [8]} \Wclose(w_i^N))
$$
By a union bound and Lemma \ref{lem:hypercube-counting} we have that $|Z| \leq (2n+8)2^p \exp(-p/8)$. Since $W^\star = W \setminus \Wclose(1) \setminus \Wclose(-1)$, this also gives that $|W^\star| \geq 2^p(1-2\exp(-p/8))$. We pick $w^\star$ according to:
$$
w^\star = \argmin_{\hat w \in W^\star} P^A_{0,w^\star}(\hat w \in Z).
$$
Putting things together and using Lemma \ref{lem:unif-oracle} gives:
\begin{align*}
    &2^p(1-2\exp(-p/8)) P_{0, w^\star}^A(w^\star \in Z) \leq |W^\star| P_{0, \hat w}^A(w^\star \in Z) \\
    &\leq \sum_{\hat w \in W^\star} P_{0,\hat w}^A(\hat w \in Z) \leq \sum_{\hat w \in W}P^A_{0,\hat w}(\hat w \in Z) \leq \sum_{\hat w \in W} p^x P^A_{0,0}(\hat w \in Z) \\
    &= p^x \sum_{\hat w \in W} P^A_{0,0}(\hat w \in Z) = p^x \sum_{\hat w \in W} \bE_{0,0}^A \left[ \mathbbm{1}[\hat w \in Z]\right] = p^x \bE_{0,0}^A\left[\sum_{\hat w \in W} \mathbbm{1}[\hat w \in Z]\right] \\
    &= p^x \bE_{0,0}^A [|Z|] \leq p^x (2n+8)2^p\exp(-p/8)
\end{align*}
Rearranging gives that
$$
P^A_{0,w^\star}( (E_n^{w^\star})^c) \leq P^A_{0,w^\star}(w^\star \in Z) \leq \frac{(2n+8)2^p\exp(-p/8)p^x}{2^p(1-2\exp(-p/8)} \leq 2(n+5)p^x \exp(-p/8) \leq \frac{1}{10},
$$
where the last line followed from our bound on $n$ (Eq. \eqref{eq:sample-comp}).
\end{proof}

We are now ready to prove Theorem \ref{thm:cube-game}. In fact, there is not much left to do. 

\begin{proof}[Proof (of Theorem \ref{thm:cube-game})]
Let the planner $A$ be sound and have an expected query cost $\bar N$, and $w^\star$ the vector from the previous lemma. Then by Markov's inequality we have:
$$
P^A_{(0,w^\star)}[N-1 \geq n] \leq \frac{ \bar N }{ n }
$$
Letting $E' = E_n^{w^\star} \cap \{ N-1 < n\}$ we have
$$
P_{0,w^\star}^A[E'] \geq (9/10)^2 - \frac{ \bar N }{n}
$$
Under event $E'$, the output of the planner satisfies $\rho(w_i^N, w^\star) \geq p/4$ for all $i \in [8]$, so the reward at the end of the game is $R < (3/4)^9$. Combined with soundness this gives
\begin{align*}
    \frac{1}{4} - 0.01 \leq f_{w^\star}(()) - 0.01 \leq \bE_{w^\star}^A[R] &\leq (\frac 3 4)^9 + (1 - P^A_{0,w^\star}[E']) \frac 3 4\\
    &\leq (\frac 3 4)^9 + (1 - (9/10)^2) \frac 3 4+ \frac{\bar N }{ n} \frac 3 4 ,
\end{align*}
which requires $\bar N > 0.02 n $, namely
$$
N > 0.02 \min\left\{ \exp(p/8)p^{-x}/16 - 5, \left(\tfrac 1 \varepsilon - 1 \right) / 7.5 \right\}
$$
Lastly, note that when $x \leq \frac{p}{16\log p}$ we have 
$$
\log n = \log (\exp(p/8)) - \log(p^{x}) = \frac p 8 - x\log p \geq \frac p 8 - \frac{p}{16\log p}\log p = \frac{p}{16}, 
$$
thus $n = \Omega(\exp(p/16))$ and in particular $N = \Omega\{2^{p \wedge K}\}$.
\end{proof}

To prove Theorem \ref{thm:reduction}, the last thing to show is that a learner which solves the MDP can be used to solve $\CubeGame$. This reduction follows exactly as in Section 4.8 of \cite{weisz2021tensorplan}

\newpage

    \section{On $\pi^\circ$ linearity}\label{app:pi-lin}
    This section shows that,  when $\pi^\circ \neq \pi^\star$, $q^\circ$ can be linear with $d-$dimensional features yet these features do not realize $\pi^\circ$-linearity.

\renewcommand{\r}{\texttt{r}}

The MDP is as follows: the states are arranged in a binary tree of length $H$. The action space is $\{\ell, \r\}$, corresponding to the $\ell$eft and $\r$ight actions. Transitions are deterministic. The reward for every $\ell$eft action is $-1$, the reward for every $\r$ight action is $+1$. See Figure \ref{fig:tree}.

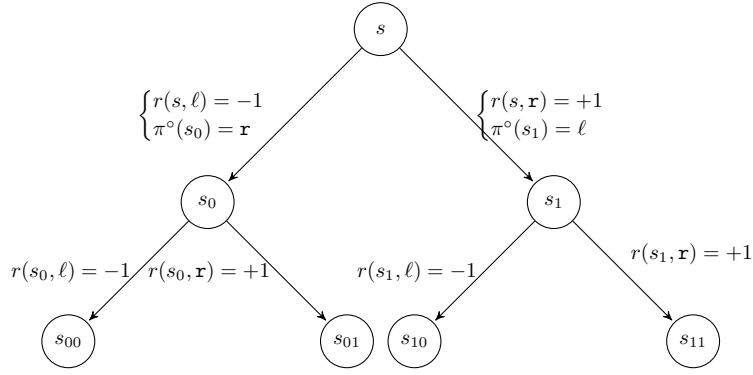
\begin{figure}[h]
    \centering
    \begin{tikzpicture}[>=stealth',node distance=1.5cm, scale=2,every node/.style={scale=0.8},shorten >=1pt,auto, every loop/.style={looseness=1}]
  \node[state] (root)      {$s$};
  
  \node[state] (0) [below left= 1in of root] {$s_{0}$};
  \node[state] (1) [below right= 1in of root] {$s_{1}$};
  
  \node[state] (00) [below left= 0.75in of 0] {$s_{00}$};
  \node[state] (01) [below right= 0.75in of 0] {$s_{01}$};
  \node[state] (10) [below left= 0.75in of 1] {$s_{10}$};
  \node[state] (11) [below right= 0.75in of 1] {$s_{11}$};
 
    \node[state,draw=none] (r00) [below = 0.3in of 00] {};
    \node[state,draw=none] (r01) [below = 0.3in of 01] {};
    \node[state,draw=none] (r10) [below = 0.3in of 10] {};
    \node[state,draw=none] (r11) [below = 0.3in of 11] {};
  
    \path[->] (root) edge [] node[scale=1,left,midway] {$\begin{cases} r(s,\ell)=-1 \\ \pi^\circ(s_0) = \r \end{cases}$} (0);
    \path[->] (root) edge [] node[scale=1,right,midway] {$\begin{cases} r(s,\texttt{r})=+1 \\ \pi^\circ(s_1) = \ell \end{cases}$} (1);
    
    \path[->] (0) edge [] node[scale=1,left,midway] {$r(s_0,\ell)=-1$} (00);
    \path[->] (0) edge [] node[scale=1,left,midway] {$r(s_0,\texttt{r})=+1$} (01);
    \path[->] (1) edge [] node[scale=1,left,midway] {$r(s_1,\ell)=-1$} (10);
    \path[->] (1) edge [] node[scale=1,midway] {$r(s_1,\texttt{r})=+1$} (11);
    
\end{tikzpicture} 
\caption{$q^\circ$-linearity does not imply $\pi^\circ$-linearity}
    \label{fig:tree}
\end{figure}

Note that we can identify every state with the action sequence that led to it (with the starting state corresponding to the empty sequence). The policy $\pi^\circ$ is defined such that, if $s = (a_0,\dots,\ell)$ then $\pi^\circ(s) = \r$ and otherwise if $s = (a_0, \dots, \r)$ then $\pi^\circ(s) = \ell$. Thus the policy will alternate the action taken at every step. This defines the $q^\circ$ function:
$$
q^\circ(s,a) = \begin{cases}
0, & \text{ if } h = 0 \text{ mod } 2 \\
-1, & \text{ if } h=1 \text{ mod } 2 \text{ and } a = \ell \\
+1, & \text{ if } h=1 \text{ mod } 2 \text{ and } a = \r
\end{cases}
$$

Note that we can linearize this in one dimension via the features $\phi^\circ(s,a) = q^\circ(s,a)$ and $\theta = 1$. However, these features do not linearly-realize $\pi^\circ$: since $\theta > 1$ then the argmax at every odd horizon will always be the right action (since $\phi^\circ(s,\r) = 1 $ and $\phi^\circ(s,\ell) = -1$).

Note that $\pi' = \argmax_a \{q^\circ(s,a) \}$, the greedy policy derived from $q^\circ$ is by definition linear with those features. For the special case where $\pi^\circ = \pi^\star$ then the greedy policy $\pi'$ lines up with the policy $\pi^\circ$, so we get linearity for free in that case.

\end{document}